\documentclass[journal]{IEEEtran}
\usepackage{amsmath,amsfonts}
\usepackage{algorithmic}
\usepackage{algorithm}
\usepackage{array}
\usepackage[caption=false,font=normalsize,labelfont=sf,textfont=sf]{subfig}
\usepackage{textcomp}
\usepackage{stfloats}
\usepackage{url}
\usepackage{verbatim}
\usepackage{graphicx}
\usepackage{cite}
\usepackage[version=4]{mhchem} % For chemistry formulae

\usepackage{array}
\usepackage{amssymb}
\usepackage{color}

\usepackage{academicons}
\usepackage{xcolor}
\usepackage{svg}

\usepackage[hidelinks]{hyperref}
\usepackage{amsthm}
\usepackage{mathtools}
\newtheorem{theorem}{Theorem}[section]
\newtheorem{corollary}{Corollary}[theorem]
\newtheorem{lemma}[theorem]{Lemma}

\usepackage[table]{xcolor}
\usepackage{multirow}
\usepackage[flushleft]{threeparttable}%To add undertext for tables

\usepackage[printonlyused]{acronym}
\acrodef{3D}{three-dimensional}
\acrodef{3GPP}{3rd generation partnership project}
\acrodef{4G}{fourth generation}
\acrodef{5G}{fifth generation}

\acrodef{B5G}{5G and beyond}
\acrodef{6G}{sixth generation}

\acrodef{DPD}{digital predistortion}
\acrodef{FC}{Fast convergence}
\acrodef{IB}{in-band}
\acrodef{ICE}{iterative compensation of error}
\acrodef{MER}{modulation error ratio}
\acrodef{NL}{nonlinear}
\acrodef{OFDM}{orthogonal frequency division multiplexing}
\acrodef{OOB}{out-of-band}
\acrodef{PA} {power amplifier}
\acrodef{PSD}{power spectral density}
\acrodef{PW-DPD}{piece-wise DPD}

\acrodef{RF} {radio frequency}
\acrodef{DVB-T2} {second generation Digital Video Broadcast-Territorial}
\acrodef{ATSC 3.0} {American next generation DTT}
\acrodef{ICE} {iterative compensation of error}
\acrodef{ILA} {indirect learning architecture}
\acrodef{FCJO} {FC joint optimization}
\acrodef{PAPR} {peak-to-average power ratio}
\acrodef{PW} {piecewise}
\acrodef{IFFT} {inverse fast Fourier transform}
\acrodef{FFT} {fast Fourier transform}
\acrodef{MP} {memory polynomial}
\acrodef{GMP} {generalized memory polynomial}
\acrodef{LUT} {look-up table}

\acrodef{ML}{machine learning}
\acrodef{OBO}{output backoff}
\acrodef{ARB}{arbitrary waveform generator}

\acrodef{NR}{new radio}
\acrodef{QAM}{quadrature amplitude modulation}
\acrodef{FLOPs}{floating point operations}
\acrodef{DSP}{digital signal processing}
\acrodef{IBO}{input back-off}
\acrodef{OBO}{output back-off}
\acrodef{VSA}{Vector Signal Analyzer}

\acrodef{NMSE}{normalized mean-squared error}

\acrodef{ILA}{indirect learning architecture}

\acrodef{mmW} {millimeter wave}
\acrodef{kNN} {k-nearest neighbors}

\acrodef{LS} {least-squares}
\acrodef{rms} {root mean square}
\acrodef{GLA} {generalized Lloyd's algorithm}
\acrodef{ITU} {International Telecommunication Union}
\acrodef{VS-GMP} {vector-switched generalized memory polynomial}
\acrodef{SVS-GMP} {statistical VS-GMP}
\acrodef{I} {in-phase}
\acrodef{Q} {quadrature}

\acrodef{ACLR}{adjacent channel leakage ratio}
\acrodef{ACPR}{adjacent channel power ratio}
\acrodef{AM/AM}{amplitude-to-amplitude}
\acrodef{AM/PM}{amplitude-to-phase}
\acrodef{AI}{artificial intelligence}
\acrodef{ANN}{artificial neural network}
\acrodef{LM}{levenberg–marquardt}

\acrodef{ReLU}{rectified linear unit}
\acrodef{Tanh}{tangent sigmoid}
\acrodef{ELU}{exponential linear unit}
\acrodef{GELU}{gaussian error linear unit}

\acrodef{CDF}{cumulative distribution function}
\acrodef{CCDF}{complementary cumulative distribution function}
\acrodef{CCRR}{computational complexity reduction ratio}

\acrodef{DT} {decision tree}
\acrodef{DT-GMP} {decision tree-based generalized memory polynomial}
\acrodef{DL}{deep learning}
\acrodef{DUT} {device under test}
\acrodef{DLA}{direct learning architecture}
\acrodef{DTT}{digital terrestrial television}
\acrodef{DDR}{dynamic deviation reduction}
\acrodef{EMP}{envelope memory polynomial}

\acrodef{EVM} {error vector magnitude}

\acrodef{FR1}{frequency range 1}
\acrodef{FR2}{frequency range 2}
\acrodef{FLOP}{floating point operation}

\acrodef{LTE}{long term evolution}
\acrodef{LTI}{linear time-invariant}

\acrodef{MLP} {multilayer perceptron}
\acrodef{MSE} {mean-squared error}

\acrodef{NN} {neural network}
\acrodef{RNN} {recurrent NN}
\acrodef{CNN} {convolutional NN}
\acrodef{STD}{standard deviation}

\acrodef{RVTDNN} {real-valued time-delay NN}
\acrodef{R2TDNN} {residual real-valued time-delay NN}
\acrodef{ARVTDNN} {augmented real-valued time-delay NN}
\acrodef{SRTDNN} {simplified real-valued time-delay NN}
\acrodef{RVTDCNN} {real-valued time-delay convolutional NN}
\acrodef{LSTM} {long short-term memory}
\acrodef{TDNN} {time-delay neural network}
\acrodef{$I/Q$} {in-phase and quadrature}
\acrodef{ResNet} {residual neural network}

\acrodef{MMSE} {minimum mean-squared error}
\acrodef{LTI} {linear time-invariant}
\acrodef{BIBO} {bounded input and bounded output}

\acrodef{LHS} {left-hand side}
\acrodef{RHS} {right-hand side}

\acrodef{MI} {memory index}
\acrodef{CRLB} {Cramér–Rao lower bound}

\acrodef{NARMA} {nonlinear auto-regressive moving average}

\acrodef{GOD} {gap-induced orthogonality delimitation}
\acrodef{i.i.d.} {independent and identically distributed}
\acrodef{PIML} {physics-informed machine learning}
\acrodef{DOB}{deviation-induced orthogonality bound}

\usepackage{siunitx}

\usepackage{pifont}% http://ctan.org/pkg/pifont
\newcommand{\cmark}{\ding{51}}%

\newtheorem{definition}{Definition}
\newtheorem*{remark}{Remark}

\newcommand*{\Scale}[2][4]{\scalebox{#1}{$#2$}}%

\usepackage{enumitem} % To start enumerate with some number

\usepackage{booktabs} % To allow toprule, midrule and bottomrule in Tables

\begin{document}

\title{Information-Theoretic Bounds and Task-Centric Learning Complexity for Real-World Dynamic Nonlinear Systems}

\author{S. S. Krishna Chaitanya Bulusu,~\IEEEmembership{Senior~Member,~IEEE,} and Mikko. J. Sillanpää,~\IEEEmembership{Member,~IEEE}
        % <-this % stops a space
%\thanks{This paper was produced by the IEEE Publication Technology Group. They are in Piscataway, NJ.}% <-this % stops a space
%\thanks{This work was supported in part by the Academy of Finland projects 6Genesis Flagship (grant number 346208) funding and Profi5 funding for Mathematics and AI: data insight for high-dimensional dynamics (HiDyn) (grant number 326291).~(\textit{Corresponding author: S. S. Krishna Chaitanya Bulusu~(email:sri.bulusu@oulu.fi)})

%S. S. Krishna Chaitanya Bulusu is with the Centre for Wireless Communications (CWC), University of Oulu, 90570 Oulu, Finland. 

%Mikko. J. Sillanpää is with the Department of Mathematical Sciences (DMS), University of Oulu, 90570 Oulu, Finland.
%\thanks{This work has been submitted to the IEEE Transactions on Signal Processing for possible publication. Copyright may be transferred without notice, after which this version may no longer be accessible.}

\thanks{S. S. Krishna Chaitanya Bulusu is with the Centre for Wireless Communications (CWC), University of Oulu, 90570 Oulu, Finland.~(\textit{Corresponding author: S. S. Krishna Chaitanya Bulusu~(email:sri.bulusu@oulu.fi)}) 

Mikko. J. Sillanpää is with the Department of Mathematical Sciences (DMS), University of Oulu, 90570 Oulu, Finland.~(\textit{(email:mikko.sillanpaa@oulu.fi)})
}
}

\markboth{THIS WORK SHALL BE SUBMITTED TO IEEE TRANSACTIONS FOR CONSIDERATION}%
{}

\maketitle

\begin{abstract}
Dynamic nonlinear systems exhibit distortions arising from coupled static and dynamic effects. Their intertwined nature poses major challenges for data-driven modeling. This paper presents a theoretical framework grounded in structured decomposition, variance analysis, and task-centric complexity bounds.

The framework employs a directional lower bound on interactions between measurable system components, extending orthogonality in inner product spaces to structurally asymmetric settings. This bound supports variance inequalities for decomposed systems. Key behavioral indicators are introduced along with a memory finiteness index. A rigorous power-based condition establishes a measurable link between finite memory in realizable systems and the First Law of Thermodynamics. This offers a more foundational perspective than classical bounds based on the Second Law.

Building on this foundation, we formulate a `Behavioral Uncertainty Principle,' demonstrating that static and dynamic distortions cannot be minimized simultaneously. We identify that real-world systems seem to resist complete deterministic decomposition due to entangled static and dynamic effects. We also present two general-purpose theorems linking function variance to mean-squared Lipschitz continuity and learning complexity. This yields a model-agnostic, task-aware complexity metric, showing that lower-variance components are inherently easier to learn.

These insights explain the empirical benefits of structured residual learning, including improved generalization, reduced parameter count, and lower training cost, as previously observed in power amplifier linearization experiments. The framework is broadly applicable and offers a scalable, theoretically grounded approach to modeling complex dynamic nonlinear systems.
\end{abstract}

\begin{IEEEkeywords}
Behavioral modeling, Complexity, Dynamic nonlinear systems, Inner product bounds, Learning complexity, Machine learning, Power amplifiers, Residual learning, Structured decomposition, System identification, Variance bounds.
\end{IEEEkeywords}

\section{Introduction}
\IEEEPARstart{D}{ynamic} nonlinear systems are widely encountered in signal processing, communications, and control applications \cite{strogatz2018nonlinear,vidyasagar2002nonlinear}. Modeling real-world dynamic nonlinear systems is challenging due to time-varying interactions among memory, nonlinearities, and external influences. A memoryless nonlinear system can be described by Taylor expansion or static regression, while a linear memory system can be approximated as time-variant and treated as an \ac{LTI} system. Dynamic nonlinear systems combine both effects, producing time-varying distortions that cannot be captured by superposition or impulse responses, and their behavior strongly depends on structure and operating conditions \cite{Schetzen_1980_NLsystems}. Further complexity arises from practical factors such as thermal drift, aging, bias instability, and environmental variations (e.g., temperature and mechanical stress), which continuously alter system response beyond static or deterministic models \cite{Boumaiza_2005_Aging}.

There is no universally accepted model for dynamic nonlinear systems, as different approaches balance accuracy, interpretability, efficiency, and scope. Common techniques include Volterra-based models \cite{Schetzen_1980_NLsystems}, block-structured models \cite{Giri2010}, neural network models \cite{NP90}, \ac{NARMA}-based methods \cite{ucak2021online}, and Kalman filter approaches \cite{Thrun2002}. Different models are preferred based on the specific system characteristics and constraints. In wireless communications, Volterra and neural network models are widely used to capture amplifier nonlinearities \cite{4118399}. In robotics and adaptive control, Kalman filters and NARMA models support real-time identification \cite{urrea2021kalman,sen2023narma}. In biomedical signal processing, block-structured models are preferred for their physiological interpretability \cite{hunter1986identification,marmarelis2004nonlinear}. This diversity underscores the challenge of accurately representing dynamic nonlinear systems across operating conditions using a single model.

Conventional block-structured models such as Wiener, Hammerstein, and Wiener–Hammerstein approximate dynamic nonlinear systems by combining LTI components with static nonlinearities \cite{Schetzen_1980_NLsystems}. These models assume strict separation between nonlinear effects and memory, an assumption often violated in practice due to environmental and operational variations \cite{pintelon2012system}. To address such limitations, \ac{AI}-based frameworks have emerged as adaptive, data-driven alternatives. Unlike conventional models with fixed structures, machine learning methods learn directly from data, enabling flexibility where memory and nonlinearity are tightly coupled. However, their high computational cost remains a bottleneck, particularly in hardware-constrained scenarios \cite{MISRA2010239}.

Our earlier work introduced an \ac{AI}-based block-structured framework for modeling dynamic nonlinear systems, demonstrated on power amplifiers in communication systems \cite{Bulusu_2025}. A \ac{PA} exemplifies a dynamic nonlinear system, with nonlinearities arising from static effects such as AM–AM and AM–PM conversion, and dynamic effects including thermal variation, bias drift, and memory \cite{wood2014behavioral}. The framework in \cite{Bulusu_2025} employs two blocks: a static nonlinear block and a neural network block. By isolating the dynamic residual and modeling only this component with a neural network, the method adapts to time-varying conditions while reducing training and inference complexity. This selective learning strategy effectively compensates for the nonlinear memory without modeling the entire distortion space. Experimental validation on a real PA showed superior performance over conventional approaches such as those in \cite{Morgan_2006_DPD_Review}.

Unlike the classical progression where theoretical results precede experimental validation, this work originated from an empirical observation in \cite{Bulusu_2025} that confirmed the abstract idea of separately learning static and dynamic effects. Our apparent empirical success raised a deeper theoretical question:
\begin{center}
\emph{Can we rigorously prove why structured residual learning works effectively in real-world dynamic nonlinear systems?}
\end{center}
This question served as the primary motivation for the present work. In seeking its resolution, we uncovered several valuable theoretical insights, which we now formalize and present in this paper.

The key contributions of this work are summarized as follows:
\begin{itemize}
\item \textbf{Deviation-induced Orthogonality Bound (DOB):}
We establish the DOB theorem in \textbf{Appendix~\ref{Appendix:DOB_Theorem}}, which provides directional lower bounds on deviations in inner product spaces under power dominance. This result generalizes classical orthogonality and forms the theoretical basis for variance and memory analysis.

\item \textbf{Diagnostic Indicators for Static–Dynamic Separation:}  
We introduce two behavioral indicators, together with an additive decomposition that conceptually separates static and dynamic effects. The `Behavioral Uncertainty Principle' (see \textbf{Section~\ref{sec:Quantifying Behavioral Uncertainty}}) shows that such a perfect separation is unattainable without residual error. Nevertheless, this abstraction enables systematic analysis of memory–nonlinearity interactions. This principle itself was an unexpected result of our framework that emerged naturally due to the distortion identity (refer \textbf{Theorem~\ref{thm:distortion_identity}}). While not pursued in detail here, this principle opens an important avenue for future research.

\item \textbf{Variance Bounds and Memory Finiteness:}
In \textbf{Section~\ref{sec:Memory_Variance}}, we introduce a measurable memory finiteness index. Using the DOB framework, we further demonstrate a fundamental link between bounded memory in realizable systems and the First Law of Thermodynamics. Then, we derive variance inequalities showing that dynamic residuals are smoother than total distortion in \textbf{Thoerem~\ref{Thm:Variance_domination_theorem}}.

\item \textbf{Weak Uncorrelatedness Lemma:}  
Although perfect orthogonality between static and dynamic residuals is not guaranteed, we establish a weak uncorrelatedness lemma in \textbf{Appendix~\ref{appendix:TheAppendixB_WeakUncorrelatednessLemma}}, supported by empirical evidence from real-world dynamic nonlinear systems, which indicates that their interaction is weakly coupled and bounded (refer \textbf{Table~\ref{realworld_table}}). While the lower bound is guaranteed by the GOD theorem, identification of a meaningful upper bound seems to be system-specific.

\item \textbf{Variance–Complexity Link and Task-Centric Learning Metric:}  
We show that the target function's variance imposes a universal lower bound on learning complexity. By linking this to mean-squared Lipschitz continuity in \textbf{Theorem~\ref{Thm:Lipschitz_dominance}}, we introduce a model-agnostic, task-aware complexity metric that explains the efficiency gains of structured residual learning in \textbf{Theorem~\ref{Thm:Learning_Complexity_Reduction_via_Structural_Decomposition}}.
\end{itemize}

This paper is organized as follows.
\textbf{Section~\ref{sec:section_2}} introduces the structured decomposition framework for dynamic nonlinear systems.
\textbf{Section~\ref{sec:section_3}} develops core theorems, behavioral indicators, and statistical identities forming a consistent foundation for real-world nonlinear system identification.
\textbf{Section~\ref{sec:The DOB Theorem and Its Fundamental Bounds on Behavior}} applies the {Deviation-induced Orthogonality Bound Theorem} to our framework, introduces the concept of behavioral uncertainty, formulates a measurable index for memory finiteness, and establishes the thermodynamic basis for the non-negativity of combined indicators, linking power boundedness to memory finiteness. It also formulates the variance domination theorem.
\textbf{Section~\ref{sec:section_5}} proposes a task-centric, model-agnostic framework linking target function variance to learning complexity via two general-purpose theorems. These results culminate in a variance-based complexity metric and explain the observed benefits of structured residual learning.
\textbf{Section~\ref{sec:section6}} provides empirical validation through PA linearization experiments, showing strong agreement between theoretical predictions and observed gains in performance and computational efficiency.
The paper concludes with key findings and directions for future work.
\begin{figure}[t!]
\centering
\includegraphics[width=8.5cm]{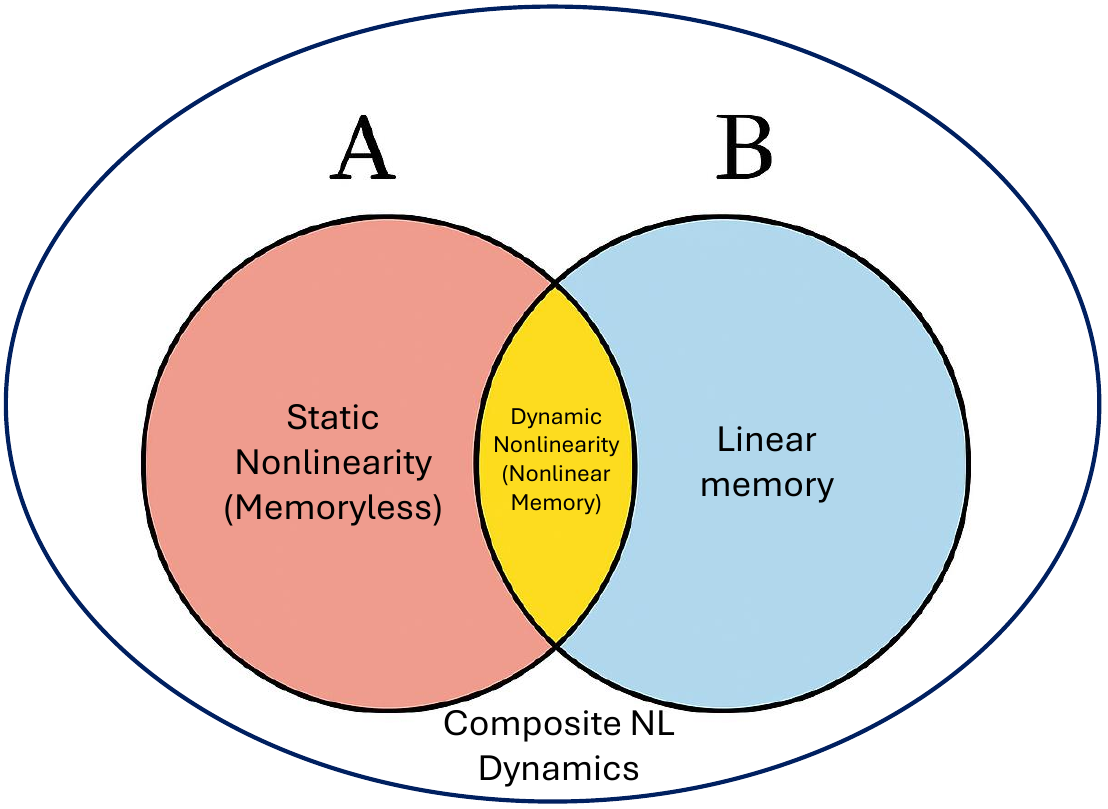}
\caption{Structured decomposition of nonlinear (NL) system behavior. $\mathbf{A}$: total nonlinearity; $\mathbf{B}$: total memory.}
\label{fig:Structured_Decomposition_Venn_Diagram}
\end{figure}
%%%%%%%%%%%%%%%%%%%%%%%%%%%%%%%%%%%%%%%%%%%%%%%%%%%%%%%%%%%%%%%%%%%%%%%
%Section
%%%%%%%%%%%%%%%%%%%%%%%%%%%%%%%%%%%%%%%%%%%%%%%%%%%%%%%%%%%%%%%%%%%%%%%
\section{Structured Decomposition Framework}\label{sec:section_2}
%%%%%%%%%%%%%%%%%%%%%%%%%%%%%%%%%%%%
% Subsub Section
%%%%%%%%%%%%%%%%%%%%%%%%%%%%%%%%%%%%
\subsection{Motivation and Decomposition Model}
Even when we set aside the inherent non-stationarity of dynamic nonlinear systems, disentangling memory effects from static nonlinearities remains a major challenge \cite{Billings2013}. These two distortion sources behave differently: static nonlinearities are time-invariant and easier to characterize, whereas dynamic effects arise from memory, hysteresis, or frequency-selective behavior and are more difficult to model and compensate.

To address this, we reformulate the standard distortion model $ Y = X + d $ into a structured decomposition:
\begin{equation}\label{StaticNLDistortion_ResidualDynamicDistortion_Segregation_Formulation}
Y = G + h,
\end{equation}
where $ X $ is the input, $ Y $ is the output, $ G $ represents the static nonlinear transformation and $ h $ is the residual dynamic distortion.
%%%%%%%%%%%%%%%%%%%%%%%%%%%%%%%%%%%%
% Subsub Section
%%%%%%%%%%%%%%%%%%%%%%%%%%%%%%%%%%%%
\subsection{A Chemist’s Perspective on Decomposition and Learning}
The overall nonlinear behavior of real systems can be viewed as a ``composite compound'' formed by interwoven static and dynamic mechanisms. Classical block models, such as Wiener and Hammerstein, attempt an orthogonal separation \cite{Giri2010}, but real-world systems rarely permit such partitions due to the inherent coupling of memory and nonlinearity.

Instead, we propose a decomposition inspired by a chemist’s practice, where a complex mixture is separated into functionally distinct components that are easier to analyze. Just as the ideal decomposition of water (\ce{H2O}) would yield hydrogen (\ce{H2}) and oxygen (\ce{O2}), yet in practice dissociation into hydrogen (\ce{H+}) ion and hydroxide ion (\ce{OH-}) is more feasible, our framework seeks a tractable, approximate separation rather than an idealized orthogonality. As illustrated in Fig.~\ref{fig:Structured_Decomposition_Venn_Diagram}, the structured decomposition yields:
\begin{enumerate}
    \item \textbf{Static nonlinearity} $ G $: Memoryless transformation capturing instantaneous distortions.
    \item \textbf{Residual dynamic distortion} $ h $: Captures all remaining effects, including both linear and nonlinear memory behavior.
\end{enumerate}
\begin{align}
\textbf{Composite nonlinear dynamics} = \underbrace{\textbf{A}}_{\text{Memoryless}} \cup \underbrace{\textbf{B}}_{\text{Memory}},
\end{align}
where $A$ is the total static nonlinear behavior and $B$ is the total memory behavior. These two are the decomposable elements of `Composite nonlinear dynamics'.

Rather than aiming for an idealized orthogonal split, this approach is computationally feasible, physically grounded, and interpretable. It supports practical goals such as dynamic modeling, real-time observation, and compensation design. Importantly, it also provides a foundation for learning: instead of training a learner to model the entire distortion $d$, we isolate the complex residual $h$, while the static part $G$ is handled analytically. This selective strategy reduces learning complexity and improves generalization. In the remainder of the paper, this decomposition underpins our variance-based analysis, theoretical bounds, and interpretability arguments.
%%%%%%%%%%%%%%%%%%%%%%%%%%%%%%%%%%%%
% Subsub Section
%%%%%%%%%%%%%%%%%%%%%%%%%%%%%%%%%%%%
\subsection{Axiom: Decomposability and Representability}\label{sec:axiom_decomposability}
The framework assumes that system input $X$ and output $Y$ can be represented as measurable signals, enabling the use of statistical learning tools. While not all natural systems (e.g., biological metabolism) admit such abstraction, for those that do, the proposed decomposition offers a viable and effective modeling paradigm.
%\subsection*{Notation Summary}
\begin{table}[t!]
\centering
\caption{Notation Summary}
\label{Notations Table}
\begin{tabular}{ll}
\toprule
Symbol & Description \\
\midrule
$ X $ & Input signal (real-valued random process) \\
$ Y $ & Output signal \\
$ G $ & Static nonlinear function of $ X $ \\
$ r = G - X $ & Static residual \\
$ h $ & Dynamic residual \\
$ d = Y - X $ & Total distortion \\
$ \mathbb{E}[\,\cdot\,] $ & Expectation (mean) operator\\
$ \operatorname{Var}(\cdot) $ & Variance operator\\
$ \operatorname{Cov}(\cdot,\cdot) $ & Covariance operator\\
$ \theta $ & Static behavior indicator (measure)\\
$ \alpha $ & Dynamic behavior indicator (measure)\\
$ \langle \cdot, \cdot \rangle $ & Inner product in Hilbert space \\
$ \|\cdot\| $ & Euclidean (2-)norm \\
$ |\cdot| $ & Modulus (absolute value) operator\\
$ \mathcal{O}(n) $ & Order of complexity \cite{cormen2009algorithms}\\
$ \Re\{\cdot\} $ & Real part of a complex number\\
$ L_f $ & Mean-squared Lipschitz constant of function $ f(x) $\\
$\sup f(x)$& Supremum of function $ f(x) $\\
dB & Decibel (logarithmic unit) \\
dBc & Decibels relative to carrier \\
\bottomrule
\end{tabular}
\end{table}

%%%%%%%%%%%%%%%%%%%%%%%%%%%%%%%%%%%%%%%%%%%%%%%%%%%%%%%%%%%%%%%%%%%%%%%%%
%Section
%%%%%%%%%%%%%%%%%%%%%%%%%%%%%%%%%%%%%%%%%%%%%%%%%%%%%%%%%%%%%%%%%%%%%%%%%
\section{Physical and Statistical Foundations for Modeling Real-World Dynamic Nonlinear Systems}
\label{sec:section_3}

This section develops results that formalize the physical and statistical constraints governing real-world dynamic nonlinear systems, spanning domains from engineering and control to finance, neuroscience, and astrophysics \cite{strogatz2018nonlinear,ljung1999system,farmer1987predicting,may1976simple}. While variance-based analysis is traditionally applied to \ac{PA}s, stock markets, and control systems, the underlying principles extend naturally to systems governed by nonlinear energy transport and bounded resource flow.

The theorems address energy conservation, bounded power flow, and static–dynamic indicators. We first establish that nonlinear systems, despite their complexity, comply with fundamental conservation and boundedness principles. We then introduce power-based diagnostic identities to separate static and dynamic effects, allowing modelers to assess the relative influence of memory and nonlinearity.

Together, these results provide a foundation for modeling strategies that respect physical laws, exploit system structure, and reduce complexity. The notations used throughout this paper are mentioned in \textbf{Table~\ref{Notations Table}}.
%%%%%%%%%%%%%%%%%%%%%%%%%%%%%%%%%%%%
% Sub Section
%%%%%%%%%%%%%%%%%%%%%%%%%%%%%%%%%%%%
\subsection{Power Boundedness and Properties of System Components}
\label{subsection:EnergyConservation_PowerBoundedness}
Although many frameworks exist for analyzing distortion and stability in nonlinear systems, no unified \textit{power-centric} law has been established. The following theorems provide such a foundation for optimizing power efficiency in nonlinear signal processing. While the initial results appear straightforward, they lead to deeper conclusions that anchor and complete the proposed framework.

\begin{lemma}[\Ac{BIBO} Lemma]
\label{lem:bounded_input_output_powers_realworldNLsystem}
Let a real-world (causal and physically realizable) dynamic nonlinear system be excited by an input signal $X$ producing an output signal $Y$. Then, the mean powers of both the input and output signals are finite, i.e.,
\begin{enumerate}
    \item $\mathbb{E}[X^2] < \infty$,  \hfill \text{(Bounded input power)}
    \item $\mathbb{E}[Y^2] < \infty$. \hfill \text{(Bounded output power)}
\end{enumerate}   
\end{lemma}

\begin{proof}[Proof of Lemma~\ref{lem:bounded_input_output_powers_realworldNLsystem}]
In all physically realizable dynamic nonlinear systems, the constituent components, whether electrical (resistors, transistors, inductors), mechanical (dampers, springs), biological (neurons), or economic (regulators), obey conservation laws and are limited by thermal effects, saturation, and material fatigue. A system driven by finite-power input cannot yield infinite-power output, since this would contradict both thermodynamic principles and practical constraints. Power boundedness is a local property of real-world systems, but it originates from the global phenomenon of energy conservation. This connection highlights the physical grounding of the boundedness results. Therefore, both input and output powers must remain finite.
\end{proof}

\vspace{1em}
\noindent\textbf{Structured Decomposition Framework.}  
Let $ X $ be a wide-sense stationary stochastic process with finite mean power. The output $ Y $ of a real-world dynamic nonlinear system admits two equivalent representations:
\begin{align}
    Y &= X + d, \quad &\text{(Distortion model)} \\
    Y &= G + h, \quad &\text{(Structured decomposition)}
\end{align}
where:
\begin{itemize}
    \item $ d $ is the total distortion,
    \item $ G $ is a memoryless static nonlinearity,
    \item $ h $ is the dynamic residual.
\end{itemize}

\begin{theorem}[Component-wise Power Boundedness]
\label{thm:bounded_power}
Under the above setup, if $\mathbb{E}[Y^2] < \infty$, then the following quantities are finite:
\begin{align*}
    \mathbb{E}[d^2] < \infty, \quad
    \mathbb{E}[G^2] < \infty, \quad
    \mathbb{E}[h^2] < \infty.
\end{align*}
\end{theorem}

\begin{proof}[Proof of Theorem~\ref{thm:bounded_power}]
\textbf{Boundedness of }$\mathbb{E}[d^2]$:  
From $Y = X + d$:
\begin{align}
\mathbb{E}[Y^2] = \mathbb{E}[X^2] + 2\mathbb{E}[Xd] + \mathbb{E}[d^2].\nonumber
\end{align}
By Cauchy-Schwarz:
\begin{align}
|\mathbb{E}[Xd]| \leq \sqrt{\mathbb{E}[X^2] \mathbb{E}[d^2]}.\nonumber
\end{align}
This implies $\mathbb{E}[d^2]$ must be finite, or else $\mathbb{E}[Y^2]$ would diverge, contradicting the lemma.

\textbf{Boundedness of }$\mathbb{E}[h^2]$:  
From $Y = G + h$:
\begin{align}
\mathbb{E}[Y^2] = \mathbb{E}[G^2] + 2\mathbb{E}[Gh] + \mathbb{E}[h^2],\nonumber
\end{align}
and similarly applying Cauchy-Schwarz, we conclude that $\mathbb{E}[h^2] < \infty$.

\textbf{Boundedness of }$\mathbb{E}[G^2]$:  
Rewriting the above:
\begin{align}
\mathbb{E}[G^2] = \mathbb{E}[Y^2] - \mathbb{E}[h^2] - 2\mathbb{E}[Gh],\nonumber
\end{align}
and since both $\mathbb{E}[Y^2]$ and $\mathbb{E}[h^2]$ are finite, so is $\mathbb{E}[G^2]$.
\end{proof}

\begin{corollary}[Structured Power Decomposition]
\label{thm:structured_power_decomposition}
Under the same setup, the following energy identities hold:
\begin{align}
    \mathbb{E}[Y^2] &= \mathbb{E}[X^2] + 2\mathbb{E}[Xd] + \mathbb{E}[d^2], \\
    \mathbb{E}[Y^2] &= \mathbb{E}[G^2] + 2\mathbb{E}[Gh] + \mathbb{E}[h^2].
\end{align}
\end{corollary}

\begin{proof}[Proof of Corollary~\ref{thm:structured_power_decomposition}]
Direct expansion from each representation using bilinearity of expectation.
\end{proof}

%%%%%%%%%%%%%%%%%%%%%%%%%%%%%%%%%%%%%%%%%%%%%%%%%%%%%%%%%%%%%%%%%%%%%%%%%
%Section
%%%%%%%%%%%%%%%%%%%%%%%%%%%%%%%%%%%%%%%%%%%%%%%%%%%%%%%%%%%%%%%%%%%%%%%%%
%%%%%%%%%%%%%%%%%%%%%%%%%%%%%%%%%%%%%%%%%%%%%%%%%%%%%%%%%%%%%%%%%%%%%%
% Definitions
%%%%%%%%%%%%%%%%%%%%%%%%%%%%%%%%%%%%%%%%%%%%%%%%%%%%%%%%%%%%%%%%%%%%%%
%%%%%%%%%%%%%%%%%%%%%%%%%%%%%%%%%
%%%%%%%%%%%%%%%%%%%%%%%%%%%%%%%%%%%%
% Subsub Section
%%%%%%%%%%%%%%%%%%%%%%%%%%%%%%%%%%%%
\subsection{Definition of Static and Dynamic Behavior Indicators }\label{subsection:Power-based_indicators_of_System_Behavior}

In real-world dynamic nonlinear systems, separating the roles of static (memoryless) and dynamic (memory-dependent) nonlinearities is challenging. Although the structural decomposition $Y = G + h$ and the distortion model $Y = X + d$ provide insight, they do not yield scalar diagnostics. To address this, we introduce two power-based measures, $\theta$ and $\alpha$, which quantify static and dynamic contributions, respectively. The measure $\theta$ captures the net power shift due to static nonlinearity, while $\alpha$ reflects dynamic interaction and serves as a sensitive indicator of memory effects. Together, they provide interpretable diagnostics for characterizing, comparing, and compensating nonlinear behavior.

\begin{definition}[Measure of Static Behavior]\label{def:theta_staticmeasure}
Let $G$ be the memoryless (static) nonlinear transformation of the input process $X$. Define the static power deviation:
\begin{align}\label{theta_defintion}
    \theta &\triangleq \mathbb{E}[G^2] - \mathbb{E}[X^2].
\end{align}
\textbf{Interpretation:} $\theta$ quantifies the net change in average power induced by the static nonlinearity relative to the input signal. It may serve as:
\begin{itemize}
    \item \textbf{A measure:} Captures the excess (or deficit) of mean-squared value from $G$ over $X$,
    \item \textbf{An indicator:} Indicates amplification when $\theta > 0$ and attenuation when $ \theta < 0 $,
    \item \textbf{A parameter:} It characterizing static shaping behavior in dynamic nonlinear system models.
\end{itemize}
\begin{remark}
    In causal systems, $\theta$ is naturally lower-bounded, preventing unphysical energy cancellation and making it a practical tool for diagnostics and optimization.    
\end{remark}
\end{definition}

\begin{definition}[Measure of Dynamic Behavior]\label{def:alpha_dynamicmeasure}
Let $ h $ be the dynamic residual, $ G $ the static nonlinearity, and $ d = Y - X $ the total distortion. Define:
\begin{align}\label{alpha_definition}
    \alpha \triangleq \mathbb{E}[hG] - \mathbb{E}[dX].
\end{align}
\textbf{Interpretation:} $ \alpha $ quantifies the net dynamic interaction power in a nonlinear system by contrasting:
\begin{itemize}
    \item the correlation between the dynamic residual $ h $ and the static nonlinearity $ G $, and
    \item the correlation between the total distortion $ d $ and the input signal $ X $.
\end{itemize}
It serves as:
\begin{itemize}
    \item \textbf{A measure:} It measures how memory effects critically influence internal energy flow and coupling between static and dynamic components in the system.
    \item \textbf{An indicator:} Indicates where a larger magnitude of $ \alpha $ suggests a more pronounced or severe dynamic interaction.
    \item \textbf{A parameter:} A suitable parameter for system identification, control, or compensation of memory effects, particularly in complex dynamic nonlinear systems.
\end{itemize}
\begin{remark}
As a diagnostic measure, $\alpha$ isolates dynamic contributions and serves as a robust indicator of memory in both linear and nonlinear systems. A nonzero $\alpha$ implies the presence of dynamic effects, whereas 
$\alpha=0$ corresponds to purely static behavior. This makes 
$\alpha$, a useful tool for observing, controlling, and compensating dynamic nonlinearities in practice.
\end{remark}
\end{definition}

%%%%%%%%%%%%%%%%%%%%%%%%%%%%
% Subsection
%%%%%%%%%%%%%%%%%%%%%%%%%%%%
\subsection{Distortion Identity and Diagnostic Relationships}
\label{subsec:distortion_diagnostics}

%%%%%%%%%%%%%%%%%%%%%%%%%%%%%%%%%%%%%%%%%%%%%%%%%%%%%%%%%%%%%%%%%%%%%%%%%
%%%% Distortion identity theorem
%%%%%%%%%%%%%%%%%%%%%%%%%%%%%%%%%%%%%%%%%%%%%%%%%%%%%%%%%%%%%%%%%%%%%%%%%
\begin{theorem}[Distortion Identity in Real-World Dynamic Nonlinear Systems]
\label{thm:distortion_identity}
Under the same setup as in Subsection~\ref{subsection:Power-based_indicators_of_System_Behavior}, where $ Y = X + d = G + h $, the following distortion identity holds:
\begin{align}\label{distortion_identity_general}
\boxed{G - X = d - h.}
\end{align}
As a result, the two quantities $ G - X $ and $ d - h $ are:
\begin{itemize}
    \item \textbf{Algebraically equivalent:} They are equal pointwise.
    \item \textbf{Statistically equivalent:} Their first and second-order statistics match:
\end{itemize}
\begin{align*}
&\textbf{(Mean identity):} \quad 
\mathbb{E}[G - X] = \mathbb{E}[d - h], \\
&\textbf{(Power identity):} \quad 
\mathbb{E}[(G - X)^2] = \mathbb{E}[(d - h)^2], \\
&\textbf{(Variance identity):} \quad 
\operatorname{Var}(G - X) = \operatorname{Var}(d - h).
\end{align*}

We interpret:
\begin{itemize}
    \item $ G - X $: the \emph{static residual}, i.e., the static deviation from the input.
    \item $ d - h $: the \emph{static-equivalent distortion}, i.e., the distortion remaining after memory-induced effects are removed.
\end{itemize}
\end{theorem}

\begin{proof}[Proof of Theorem~\ref{thm:distortion_identity}]
Starting with $ Y = G + h = X + d $, we rearrange:
\begin{align}
G - X = d - h.\nonumber
\end{align}
This proves pointwise equality. Taking expectations:
\begin{align}\label{distortion_identity_mean}
\mathbb{E}[G - X] = \mathbb{E}[d - h],
\end{align}
and
\begin{align}\label{distortion_identity_mean_power}
\mathbb{E}[(G - X)^2] = \mathbb{E}[(d - h)^2].
\end{align}
The variances also match since:
\begin{align}\label{distortion_identity_variance}
\operatorname{Var}(G - X) = \mathbb{E}[(G - X)^2] - (\mathbb{E}[G - X])^2 = \operatorname{Var}(d - h).
\end{align}
\end{proof}

\begin{remark}
This identity is fundamental to nonlinear system analysis, enabling transitions between static-residual and memory-compensated distortion views. In particular, the power identity, i.e. \eqref{distortion_identity_mean_power}, serves as a cornerstone for power-based diagnostics and variance decomposition.
\end{remark}

%%%%%%%%%%%%%%%%%%%%%%%%%%%%%%%%%%%%%%%%%%%%%%%%%%%%%%%%%%%%%%%%%%%%%%%%%
%%%% Residual decomposition corollary
%%%%%%%%%%%%%%%%%%%%%%%%%%%%%%%%%%%%%%%%%%%%%%%%%%%%%%%%%%%%%%%%%%%%%%%%%
\begin{corollary}[Additive Residual Decomposition of Total Distortion]
\label{cor:static_dynamic_residual_decomposition}
Under the same setup as above, the total distortion admits the following additive residual decomposition:
\begin{equation}
\label{eq:distortion_decomposition} 
\begin{aligned} 
d = \underbrace{G - X}_{\text{static residual}} + \underbrace{h}_{\text{dynamic residual}}.
\end{aligned} 
\end{equation}
This decomposition holds pointwise for all inputs.
\end{corollary}

\begin{remark}
The static residual $G - X $ represents memoryless nonlinear effects, while the dynamic residual $h$ accounts for memory-driven components. This decomposition supports power-based separation and underscores the diagnostic value of the cross-term $ \mathbb{E}[(G - X)h]$.
\end{remark}

%%%%%%%%%%%%%%%%%%%%%%%%%%%%%%%%%%%%%%%%%%%%%%%%%%%%%%%%%%%%%%%%%%%%%%%%%
%%%% Diagnostic identity theorem
%%%%%%%%%%%%%%%%%%%%%%%%%%%%%%%%%%%%%%%%%%%%%%%%%%%%%%%%%%%%%%%%%%%%%%%%%
\begin{theorem}[Diagnostic Identity for Static-Dynamic Interaction]
\label{thm:diagnostic_identity}
Under the same setup, define:
\begin{align*}
\theta &\triangleq \mathbb{E}[G^2] - \mathbb{E}[X^2], \\
\alpha &\triangleq \mathbb{E}[hG] - \mathbb{E}[dX], \\
r &\triangleq G - X = d - h.
\end{align*}
Then the following diagnostic identity holds:
\begin{align}
\boxed{
\theta + 2\alpha =  \mathbb{E}[r^2] + 2\,\mathbb{E}[rh].\nonumber
}
\end{align}
\end{theorem}

\begin{proof}[Proof of Theorem~\ref{thm:diagnostic_identity}]
From $ r = G - X = d - h $, we write $ G = r + X $ and $ d = r + h $. Then:

\begin{align*}    
\theta + 2\alpha =& \left( \mathbb{E}[G^2] - \mathbb{E}[X^2] \right) + 2\left( \mathbb{E}[hG] - \mathbb{E}[dX] \right) \\=& \mathbb{E}[G^2] - \mathbb{E}[X^2] + 2\mathbb{E}[hG] - 2\mathbb{E}[dX].
\end{align*}
Recall that $ r = G - X = d - h $, we write $ G = r + X $ and $ d = r + h $ in the above expression:
\begin{align*}    
\theta + 2\alpha =& \bigg(\bigl(\underbrace{\mathbb{E}[X^2] +2\mathbb{E}[rX] + \mathbb{E}[r^2]}_{\mathbb{E}[G^2]}\bigr) - \mathbb{E}[X^2]\bigg) + \\&2\bigg(\bigl(\underbrace{\mathbb{E}[hr]+\mathbb{E}[hX]}_{\mathbb{E}[hG]}\bigr)-\bigl(\underbrace{\mathbb{E}[rX]+\mathbb{E}[hX]}_{\mathbb{E}[dX]}\bigr)\bigg).\\
=&\mathbb{E}[r^2] +2\mathbb{E}[hr]. 
\end{align*}
\end{proof}

\begin{remark}
This identity reveals how the deviation or gap between total and memory-induced distortion is filled by static residual power and its dynamic interaction. The cross-term $ \mathbb{E}[(G - X)h] $ becomes critical in determining whether memory cancels or reinforces static effects.
\end{remark}

%%%%%%%%%%%%%%%%%%%%%%%%%%%%%%%%%%%%%%%%%%%%%%%%%%%%%%%%%%%%%%%%%%%%%%%%%
% SECTION
%%%%%%%%%%%%%%%%%%%%%%%%%%%%%%%%%%%%%%%%%%%%%%%%%%%%%%%%%%%%%%%%%%%%%%%%%
\section{The DOB Theorem and Its Fundamental Bounds on Behavior}\label{sec:The DOB Theorem and Its Fundamental Bounds on Behavior}
{\color{black}
%%%%%%%%%%%%%%%%%%%%%%%%%%%%%%%%%%
%Subsection
%%%%%%%%%%%%%%%%%%%%%%%%%%%%%%%%%%
\subsection{The DOB Theorem in Inner Product Spaces}\label{sec:DOB_Theorem_Section4.1}
At the foundation of our framework lies a universal inequality in inner product spaces, termed the \emph{Deviation-induced Orthogonality Bound (DOB) Theorem} (see \textbf{Theorem~\ref{thm:DOB_IP_space}} in Appendix~\ref{Appendix:DOB_Theorem}). The theorem establishes a nontrivial lower bound on the inner product between a deviation gap $R = A - B$ and either element $A$ or $B$, under the condition that power dominance ($ \|A\|^2 \geq \|B\|^2$) holds.

This theorem does not assume any statistical properties (e.g., Gaussianity, ergodicity) and is purely geometric in nature. For real-valued systems, one key bound:
\begin{align}
\langle B, R \rangle \geq -\frac{1}{2} \|R\|^2.
\end{align}
This directional inequality plays a central role in our analysis: it underpins the Weak Uncorrelatedness Lemma (\textbf{Lemma~\ref{lemma1}}), supports the Variance Domination Theorem (\textbf{Theorem~\ref{Thm:Variance_domination_theorem}}), and leads to the bound$ \theta + 2\alpha > 0 $, which connects power dominance to memory finiteness and variance dominance to mean-squared Lipschitz continuity (\textbf{Theorem~\ref{Thm:Lipschitz_dominance}}). A detailed discussion and proof are provided in \textbf{Appendix~\ref{Appendix:DOB_Theorem}}.
%%%%%%%%%%%%%%%%%%%%%%%%%%%%%%%%%%
%Subsection
%%%%%%%%%%%%%%%%%%%%%%%%%%%%%%%%%%
\subsection{The Thermodynamic Basis for Behavioral Constraints: The Non-negativity of Combined Behavior Indicators}
\label{sec:Quantifying Behavioral Uncertainty}
The first bound of \textbf{Theorem~\ref{thm:DOB_IP_space}} gives rise to the \textbf{Behavioral Uncertainty Principle} (BUP) in dynamic nonlinear systems. This principle captures a fundamental constraint on the interaction between static and dynamic components, a constraint directly mandated by the laws of thermodynamics. By behavioral uncertainty, we refer to the inherent and irreducible limits faced by a \emph{General Learning System} (\textbf{Definition~\ref{def:general_learning_system}}) in observing, predicting, or controlling interactions among a system’s internal signal components, even with complete knowledge of its governing equations and statistical properties.
%%%%%%%%%%%%%%%%%%%%%%%%%%%%%%%%%%
%Subsub Section
%%%%%%%%%%%%%%%%%%%%%%%%%%%%%%%%%%
\subsubsection{The Non-negativity Constraint from First Principles}
\label{sec:Derivation of Nonnegativity of Theta+2Alpha}
In \textbf{Section~\ref{subsection:Power-based_indicators_of_System_Behavior}}, we established the diagnostic identity for static-dynamic interaction in \textbf{Theorem~\ref{thm:diagnostic_identity}}. This identity links the intrinsic behavior indicators $ \theta $ and $ \alpha $ to the difference in mean power between the total distortion $ d $ and its dynamic residual $ h $.

The First Law of Thermodynamics, which dictates conservation of energy, imposes a fundamental constraint on all physically realizable systems: a system cannot indefinitely generate its own energy. In signal-theoretic terms, this means the mean power of the total distortion, $\mathbb{E}[d^2]$, must be greater than the mean power of the memory-induced residual, $\mathbb{E}[h^2]$. This constraint is formally guaranteed in our framework by the \textbf{Theorem~\ref{thm:DOB_IP_space}}, when applied with the assignment $ A \triangleq d $, $ B \triangleq h $, and $ R \triangleq d - h = G - X $.\footnote{We exclude the degenerate case $ \mathbb{E}[d^2] = \mathbb{E}[h^2] $; it corresponds to a pure dynamic linear system with $ G = 0 $, an idealization not observed in real-world physical systems.}

\paragraph{Proposition 1 (Thermodynamic Non-negativity of Combined Behavior Indicators)}\label{Proposition1_Causality}
For any causal and physically realizable dynamic nonlinear system, given the diagnostic identity from \textbf{Theorem~\ref{thm:diagnostic_identity}}, the sum of the static behavior indicator $ \theta $ and twice the dynamic behavior indicator $ \alpha $ satisfies:
\begin{align}
\boxed{\theta + 2\alpha > 0.}
\label{Nonnegativity_Condition_of_theta_plus_2alpha}
\end{align}
This constraint is a direct consequence of the First Law of Thermodynamics, rigorously demonstrated by the \textbf{Theorem~\ref{thm:DOB_IP_space}}.

\paragraph{The Nuance of Residual Coupling}
For real-world dynamic nonlinear systems, assuming $ \mathbb{E}[rh] = 0 $ is often inadequate. The Weak Uncorrelatedness Lemma (see \textbf{Appendix~\ref{appendix:TheAppendixB_WeakUncorrelatednessLemma}}) provides a relaxed coupling bound, $\left| \mathbb{E}\left[(G - X)\, h\right] \right| \leq \epsilon$, motivated by time-scale separation. %However, tighter and universal upper bounds on this interaction remain a critical challenge and often necessitate empirical analysis for specific system classes.
%%%%%%%%%%%%%%%%%%%%%%%%%%%%%%%%%%
%Subsubsection
%%%%%%%%%%%%%%%%%%%%%%%%%%%%%%%%%%
\subsubsection{Interpretation and Diagnostic Implications}
\label{sec:Interpretation of Nonnegativity of Theta+2Alpha}
The \textbf{Behavioral Uncertainty Principle}, underpinned by the thermodynamic constraint $ \theta + 2\alpha > 0 $, reflects several key insights into real-world dynamic nonlinear systems.

First, it demonstrates a fundamental trade-off: \emph{a system cannot simultaneously minimize both the dynamic residual $ h $ and the static residual $ G - X $}. Any attempt to reduce one will inherently affect the other. This is captured in the inequality:
\begin{align}
\boxed{
\mathbb{E}\bigl[h(G - X)\bigr]
\geq -\frac{1}{2}\mathbb{E}\bigl[(G - X)^2\bigr].}
\label{eq:lower_bound_static_dynamic_interaction}
\end{align}
This bound, derived from the \textbf{Theorem~\ref{thm:DOB_IP_space}}, expresses a structural limitation on the interaction of distortions in the system. Efforts to perfectly learn or eliminate one component necessarily increase uncertainty in the other—posing an intrinsic limit to any \emph{General Learning System}'s capabilities. \emph{Real-world systems seem to resist complete deterministic decomposition due to entangled nonlinear and memory effects.}

Second, the principle implies a bounded level of \emph{destructive interference} between the static and dynamic residuals. That is, dynamic behavior cannot fully cancel static mismatch. The lower bound in \eqref{eq:lower_bound_static_dynamic_interaction} ensures that such compensation is energetically constrained.

Violating this bound would imply $ \theta + 2\alpha < 0 $, which contradicts the First Law of Thermodynamics. Such a violation would correspond to pathological scenarios such a:
\begin{itemize}
\item A physically unrealizable system, which must be generating internal energy to compensate for the imbalance.
\item Non-causal memory, where the system responds to future input, exhibiting anticipatory behavior.
\item A modeling artifact, such as a black-box learning model that has overfit to data and does not represent a physical system.
\end{itemize}

Thus, this constraint acts as both a diagnostic tool and a fundamental boundary in modeling and analyzing real-world dynamic nonlinear systems—akin to the role of uncertainty principles in physics that limit the simultaneous precision of conjugate variables.

In fact, the applicability of the \textbf{Theorem~\ref{thm:DOB_IP_space}} can be extended more broadly: \emph{for any two measurable and coupled entities within a physically realizable system, if one exhibits power dominance over the other, their mutual interaction is structurally constrained from below. This holds irrespective of their temporal ordering, causal connection, or spatial separation. As long as both entities reside in a common inner product space—mathematically or physically—the system imposes a fundamental limitation on their complete decoupling or joint observability.} This constraint explains why perfect \ac{MMSE}-style orthogonality, as assumed in Wiener and Hammerstein models, is generally unattainable in realistic systems with coupling or power imbalance. Even under equal power, the bound persists in a limiting form, allowing only narrow conditions for such orthogonality to hold.
%%%%%%%%%%%%%%%%%%%%%%%%%%%%%%%%
% Subsection
%%%%%%%%%%%%%%%%%%%%%%%%%%%%%%%%
\subsection{Characterizing Memory Finiteness and Variance in Real-world dynamic NL Systems}\label{sec:Memory_Variance}
Building upon the non-negativity constraint of combined behavior indicators (\textbf{Proposition~\ref{Proposition1_Causality}}), we now present quantitative tools and bounds that further characterize memory effects and the statistical structure of system residuals. These results extend the Behavioral Uncertainty Principle to concrete diagnostic indices and variance relations. The profound implication of our work is that memory in physically realizable systems is inherently finite, and this finiteness is a direct consequence of energy conservation.
%%%%%%%%%%%%%%%%%%%%%%%%%%%%%%%%%%
%Subsubsection
%%%%%%%%%%%%%%%%%%%%%%%%%%%%%%%%%%
\subsubsection{The Memory Finiteness Index (MFI)}
\label{sec:Memory Finiteness Index}
Memory in physically realizable dynamic nonlinear systems is inherently finite: the effect of past inputs must eventually decay. This observation motivates a quantitative index to capture the system’s memory behavior in energy terms.
\paragraph{Proposition 2 (Memory Finiteness Index)}
\textit{Under the same setup as Proposition~\ref{Proposition1_Causality},} we define the Memory Finiteness Index (MFI), denoted by $ \mathcal{M} $, as:
\begin{equation}
\label{eq:memory_finiteness_index}
\boxed{
\mathcal{M} = 1 - \frac{\mathbb{E}[h^2]}{\mathbb{E}[d^2]} = \frac{\theta + 2\alpha}{\mathbb{E}[d^2]}, \quad \text{for } \mathbb{E}[d^2] \neq 0.
}
\end{equation}
This index quantifies the proportion of distortion power attributable to static effects.
%%%%%%%%%%%%%%%%%%%%%%%%%%%%%%%%%%
%Paragraph
%%%%%%%%%%%%%%%%%%%%%%%%%%%%%%%%%%
\paragraph{Interpretation and Diagnostic Significance}
While the foundational thermodynamic bounds in information theory, such as those proposed by Landauer \cite{landauer1961irreversibility} and Bekenstein \cite{bekenstein1981universal}, originate from the Second Law of Thermodynamics, our framework establishes a measurable, first-principles connection between the \textbf{First Law} and memory finiteness in structurally decomposable real-world dynamic NL systems.\footnote{In this paper, the First Law refers to energy conservation: energy cannot be created or destroyed, only transformed. While energy conservation is a global law applying to closed systems, power boundedness is a local constraint often imposed in signal processing and system theory. Engineers typically work with power-bounded signals, whether finite-energy signals (e.g., pulses) or finite-power signals (e.g., stationary processes).} %The knowledge we gain today is built upon the invaluable contributions of these past stalwarts.

In causal, physically realizable systems with $ \mathbb{E}[d^2] > 0 $, \emph{the First Law of Thermodynamics and \textbf{Theorem~\ref{thm:DOB_IP_space}} rigorously ensure} $ \theta + 2\alpha > 0 $, implying that the index must satisfy:
\begin{align}
\boxed{ \mathcal{M} \in (0,1].}\nonumber
\end{align}
\begin{itemize}
\item $ \mathcal{M} \approx 1 $: Indicates $ h \approx 0 $, meaning distortion is predominantly static; signifying strong memory finiteness.
\item $ \mathcal{M} \approx 0 $: Suggests $ h \approx d $, meaning distortion is largely dynamic; indicating long memory.
\end{itemize}

For a structurally decomposable system to be physically realizable, the First Law's principle of energy conservation dictates a crucial constraint: the mean power of the total distortion, $\mathbb{E}[d^2]$, must be greater than the mean power of the memory-induced residual, $\mathbb{E}[h^2]$.

The \textbf{DOB theorem} serves as the mathematical tool to formalize this physical mandate. By applying the theorem with $A=d$ and $B=h$, its core assumption of power dominance, $\mathbb{E}[d^2] > \mathbb{E}[h^2]$, is shown to be a direct consequence of the First Law in this framework.
This leads to the logical equivalence that energy conservation in a physically realizable system is satisfied if and only if the total distortion power exceeds the dynamic residual power, which in turn is equivalent to the condition $\theta+2\alpha>0$. This final condition, $\theta+2\alpha>0$, directly implies that $\mathcal{M}$, defined as in \eqref{eq:memory_finiteness_index}, must be a strictly positive and finite-valued quantity. Since a physically realizable system must satisfy this condition, it means that all such systems are inherently \textbf{memory finite}. Thus, the framework establishes a new, first-principles connection between a fundamental physical law and a measurable property of a system.
\begin{align}
\boxed{\Scale[0.95]{\text{Energy Conservation} \iff \theta+2\alpha > 0 \iff \text{Finite Memory}.}}
\end{align}
%%%%%%%%%%%%%%%%%%%%%%%%%%%%%%%%
% Subsubsection
%%%%%%%%%%%%%%%%%%%%%%%%%%%%%%%%
\subsubsection{Variance Domination in Real-World Dynamic Nonlinear Systems}
\label{sec:Variance Domination}

\begin{theorem}[Variance Domination Theorem]
\label{Thm:Variance_domination_theorem}
\textit{Under the same setup,} let $ d = Y - X $ and $ r = G - X $. Then the dynamic residual variance is strictly bounded by the total distortion variance:
\begin{equation}
\label{eq:residual_variance_reduction}
\boxed{ \operatorname{Var}(h) < \operatorname{Var}(d).}
\end{equation}
\end{theorem}

\begin{proof}[Proof of Theorem~\ref{Thm:Variance_domination_theorem}]
We analyze two cases: zero-mean and general.

\textbf{Case 1: Zero-mean setting.}  
Assume $\mathbb{E}[X] = \mathbb{E}[Y] = 0$, so all derived terms are also zero-mean.

Define the static residual as $r = G - X$, so that:
\begin{align}
d = Y - X = G + h - X = r + h.
\end{align}
Applying the variance decomposition identity:
\begin{align}
\operatorname{Var}(d) = \operatorname{Var}(h) + \operatorname{Var}(r) + 2\,\mathbb{E}[hr].
\end{align}

To prove $\operatorname{Var}(h) < \operatorname{Var}(d)$, it suffices to show:
\begin{equation}\label{eq:VDT_Condition_To_be_proved}
\operatorname{Var}(r) + 2\,\mathbb{E}[hr] > 0.
\end{equation}

Invoking the \textbf{DOB Theorem} (refer \textbf{Appendix~\ref{Appendix:DOB_Theorem}}) with $B \triangleq h$ and $R \triangleq r = d - h$, we have:
\begin{align}
\mathbb{E}[hr] \geq -\frac{1}{2} \mathbb{E}[r^2] \Rightarrow \mathbb{E}[r^2] + 2\,\mathbb{E}[hr] \geq 0.
\end{align}

The equality occurs only when both $r = 0$ and $h = 0$, which implies $d = 0$ (i.e., a memoryless, distortionless system). In all other realizable systems with nonzero residuals, the inequality is strict, proving the result.

Additionally, by the \textbf{Weak Uncorrelatedness Lemma} (refer \textbf{Appendix~\ref{appendix:TheAppendixB_WeakUncorrelatednessLemma}}), we have:
\begin{align}
|\mathbb{E}[hr]| \leq \epsilon, \quad \text{with } \epsilon > 0.
\end{align}
Thus, $\mathbb{E}[r^2] + 2\,\mathbb{E}[hr] > 0$ for all practical systems.\\
\textbf{Case 2: Arbitrary mean}.

Let $ d' = d - \mathbb{E}[d] $, $ h' = h - \mathbb{E}[h] $, $ r' = r - \mathbb{E}[r] $. Variance is invariant under mean shift, and the same decomposition holds for mean-centered variables:
\begin{align}
\operatorname{Var}(d') = \operatorname{Var}(h') + \operatorname{Var}(r') + 2\mathbb{E}[h'r'].\nonumber
\end{align}
The inequality thus holds in the general case.
\end{proof}

In practice, many systems (e.g., PAs, climate models) exhibit local stationarity. Thus, this theorem remains applicable in windowed or adaptive diagnostic contexts.
}%end of red color
%%%%%%%%%%%%%%%%%%%%%%%%%%%%%%%%%%%%%%%%%%%%%%%%%%%%%%%%%%%%%%%%%%%%%%%%%%
% Section
%%%%%%%%%%%%%%%%%%%%%%%%%%%%%%%%%%%%%%%%%%%%%%%%%%%%%%%%%%%%%%%%%%%%%%%%%%
\section{Information-Theoretic Principles for Structured Modeling of Nonlinear Systems and a Complexity Measure Based on Target Function Variance}\label{sec:section_5}
%%%%%%%%%%%%%%%%%%%%%%%%%%%%%%%%%%
%Subsection
%%%%%%%%%%%%%%%%%%%%%%%%%%%%%%%%%%
\subsection{Variance as a Measure of Informational Uncertainty}
In this section, we develop a framework connecting the variance of a target function with its mean-squared Lipschitz continuity and, ultimately, the computational effort required to learn it. Variance serves as a surrogate for informational uncertainty: functions with larger variance are more irregular and demand more information to approximate within a given error tolerance. Although our formulation does not invoke Shannon entropy, it is consistent with the principle that unpredictability corresponds to informational richness. This motivates the use of information-theoretic bounds in our framework for structured modeling and residual learning.

Classical complexity measures such as VC dimension~\cite{vapnik1998statistical}, Rademacher complexity~\cite{bartlett2002rademacher}, norm-based bounds~\cite{neyshabur2015norm}, and kernel methods~\cite{jacot2018ntk} emphasize hypothesis class capacity but largely ignore the structure of the target function. Our approach shifts focus to intrinsic variability: higher variance, and thus larger mean-squared Lipschitz constants, implies greater learning difficulty under fixed error tolerance. This task-centric view holds regardless of model class or algorithm.

The idea recalls Kolmogorov complexity~\cite{kolmogorov1965three}, which characterizes an object by the length of its shortest description. While uncomputable in practice, it motivates our variance-based diagnostic, which is tractable under finite data and noise. By quantifying irregularity through variance and mean-squared Lipschitz continuity, we obtain a measurable indicator of learnability.

Our first theorem shows that, under mild assumptions, the distortion component $d(x)$ of a structured nonlinear system has a strictly larger mean-squared Lipschitz constant than the residual $h(x)$, implying smoother residuals. The second theorem demonstrates that smoother functions reduce learning complexity in terms of parameter efficiency, convergence, generalization, and computational cost. These results justify decomposing nonlinear systems into static and dynamic components, restricting learning to the residual.

This also clarifies the success of our hybrid residual learning strategy in power amplifier linearization~\cite{Bulusu_2025}, where classical methods modeled the static part $G$ and neural networks handled the residual $d$. The observed gains in performance and compactness align with the prediction that smoother residuals are easier to learn.

To our knowledge, this is among the first systematic efforts to link learning complexity to variance via mean-squared Lipschitz constants. Classical Poincaré-type inequalities connect variance and Lipschitz continuity for individual functions but do not extend to function comparisons. The mean-squared Lipschitz constant, defined as the expected squared gradient magnitude, resolves this by aligning more directly with variance and enabling consistent complexity comparisons.

The framework is \emph{architecture-independent}, \emph{learning-agnostic}, and broadly applicable in regression, control, and signal transformation. It complements Probably approximately correct-Bayes bounds~\cite{mcallester1999pac}, Barron-type results~\cite{barron1993universal}, and minimum description-length principles~\cite{grunwald2007minimum}, by providing a measurable, task-level diagnostic of learning difficulty.

%%%%%%%%%%%%%%%%%%%%%%%%%%%%%%%%%%
%Subsection
%%%%%%%%%%%%%%%%%%%%%%%%%%%%%%%%%%
% %%%%%%%%%%%%%%%%%%%%%%%%%%%%%%%%%%%%%%%%%%%%%%%%%%%%%%%
%modified theorem
\subsection{Implication of Mean-Squared Lipschitz Dominance from Variance Dominance}
{\color{black}
\begin{theorem}[Mean-Squared Lipschitz Dominance in Systems] \label{Thm:Lipschitz_dominance}
Let $ x $ be a random variable (or random vector) with finite variance, defined on a probability space. This $ x $ represents the input or state of a system and may be a realization or a snapshot derived from a dynamic random process $X$.

Consider functions $ d(x) $, $ h(x) $, and $ r(x) $ such that the total function $ d(x) $ decomposes as:
\begin{align}
d(x) = r(x) + h(x), 
\end{align}
where $ r(x) = G(x) - x $ for some function $ G(x) $. (The terms ``static residual" and ``dynamic residual" are descriptive labels to categorize the components.)

Suppose the following conditions hold:
\begin{enumerate} 
\item [\textbf{(A1)}] \textbf{Square-integrability:} $ x $, $ d(x) $, $ h(x) $, and $ r(x) $ are square-integrable random variables, meaning their second moments (and thus variances) are finite. This ensures all relevant expectations in the subsequent definitions and proof are well-defined.

 \item [\textbf{(A2)}] \textbf{Variance Dominance:}
 \begin{align*} 
 \operatorname{Var}\bigl(d(x)\bigr) > \operatorname{Var}\bigl(h(x)\bigr).
\end{align*}
\end{enumerate}

Define the mean-squared Lipschitz constant $ L_f $ for a function $ f(x) $ as:
\begin{align}
L_f^2 \triangleq \sup_{x_1 \neq x_2} \frac{\mathbb{E}\bigl[\|f(x_1) - f(x_2)\|^2\bigr]}{\mathbb{E}\bigl[\|x_1 - x_2\|^2\bigr]}
\end{align}
where $ x_1 $ and $ x_2 $ are `\ac{i.i.d.} samples of $ x $'. It is also assumed that $\mathbb{E}\bigl[\|x_1 - x_2\|^2\bigr] > 0$ (i.e., $ x_1 $ and $ x_2 $ are not always equal).

\noindent Then:
\begin{align}
\boxed{L_d > L_h.}\nonumber
\end{align}
\end{theorem}
%%%%%%%%%%%%%%%%%%%%%%%%%%%%%%%%%%%%%%%%%%%%%%%%%%%%%%%
% MODIFIED PROOF
 \begin{proof}[Proof of Theorem~\ref{Thm:Lipschitz_dominance}]
 Given the Assumptions~1 and 2 of the theorem (finite variance of distortion for $h(x)$ and $d(x)$, it follows directly from \textbf{Theorem~\ref{thm:LipschitzContinuitySufficiency}} (Sufficient Condition for Mean-Squared Lipschitz Continuity) that both ensures that both $ h(x) $ and $ d(x) $ are mean-squared Lipschitz continuous. Thus, their respective constants $ L_h $ and $ L_d $ are finite.
 \\
 The total distortion $d$ is decomposes as:
\begin{align}
d(x) = r(x) + h(x).\nonumber
\end{align}
Hence, for any pair of inputs $ x_1\neq x_2 $, define:
\begin{equation}\label{eq:Delta_x_d_h_r_definitions}
\begin{aligned}
    \Delta X&=x_1-x_2,~\Delta r=r(x_1)-r(x_2)=\Delta d-\Delta h,\\
    \text{where~}\Delta d&=d(x_1)-d(x_2),~\text{and~}\Delta h=h(x_1)-h(x_2).
\end{aligned}
\end{equation}
Expanding the mean-squared difference of the total distortion:
\begin{align}\label{eq:Square of Lipshitz constant of d expansion_1}
    \mathbb{E}\bigl[|\Delta d|^2\bigr]&=\mathbb{E}\bigl[|\Delta r+\Delta h|^2\bigr],\nonumber\\
    &=\mathbb{E}\bigr[|\Delta r|^2\bigr]+\mathbb{E}\bigl[|\Delta h|^2\bigr]+2\mathbb{E}\bigl[|\Delta r\Delta h|\bigr].
\end{align}
 Dividing both sides of  \eqref{eq:Square of Lipshitz constant of d expansion_1} by $\mathbb{E}\bigl[|\Delta x|\bigr]^2$, and taking the supremum over all $ x_1\neq x_2 $, we get:
 \begin{align}
     %L_d^2=&\sup_{x_1\neq x_2 }\Scale[1.0]{\Bigg(\frac{\mathbb{E}\bigr[|\Delta r|^2\bigr]}{\mathbb{E}\bigl[|\Delta x|\bigr]^2}+\frac{\mathbb{E}\bigl[|\Delta h|^2\bigr]}{\mathbb{E}\bigl[|\Delta x|\bigr]^2}+2\frac{\mathbb{E}\bigl[|\Delta r\Delta h|\bigr]}{\mathbb{E}\bigl[|\Delta x|\bigr]^2}\Bigg)},\label{eq:Square of Lipshitz constant of d expansion_2_first}\\
     L_d^2=& \sup_{x_1\neq x_2 }{\Bigg(\frac{\mathbb{E}\bigr[|\Delta r|^2\bigr]}{\mathbb{E}\bigl[|\Delta x|\bigr]^2}+L_h^2+2\frac{\mathbb{E}\bigl[|\Delta r\Delta h|\bigr]}{\mathbb{E}\bigl[|\Delta x|\bigr]^2}\Bigg)},\label{eq:Square of Lipshitz constant of d expansion_2_second}
 \end{align}
 where, by definition, $L_h^2=\sup_{x_1\neq x_2 }\bigg(\frac{\mathbb{E}\bigl[|\Delta h|^2\bigr]}{\mathbb{E}\bigl[|\Delta x|\bigr]^2}\bigg)$.
 
Under Assumption (i.e. i.i.d. samples of $x_1$ and $x_2$, for any function $f(x)$):
\begin{align}
    \mathbb{E}\bigl[|\Delta f|^2\bigr] &= \mathbb{E}[|f(x_1) - f(x_2)|^2],\\
    &=\mathbb{E}[f(x_1)^2 - 2f(x_1)f(x_2) + f(x_2)^2],\\
    &= \mathbb{E}[f(x_1)^2] - 2\mathbb{E}[f(x_1)f(x_2)] + \mathbb{E}[f(x_2)^2].
\end{align}
 Since $x_1$ and $x_2$ are i.i.d., $f_1$ and $f_2$ are also i.i.d. and thus statistically independent. Therefore, $\mathbb{E}[f(x_1)f(x_2)] = \mathbb{E}[f(x_1)]\mathbb{E}[f(x_2)] = (\mathbb{E}[f])^2$. Also, $\mathbb{E}[f(x_1)^2] = \mathbb{E}[f(x_2)^2] = \mathbb{E}[f^2]$.
Substituting these into the expression:
\begin{align}
\mathbb{E}[|\Delta f|^2] &= \mathbb{E}[f^2] - 2(\mathbb{E}[f])^2 + \mathbb{E}[f^2], \nonumber\\
&= 2(\mathbb{E}[f^2] - (\mathbb{E}[f])^2), \nonumber\\
&= 2\operatorname{Var}(f).\label{eq:Meanpower of Delta Y in terms of variance of Y}
\end{align}
In a similar fashion, we apply the relationship from \eqref{eq:Meanpower of Delta Y in terms of variance of Y} to our functions $d(x)$ and $h(x)$
\begin{align}
    \mathbb{E}\bigl[|\Delta d|^2\bigr] = 2\operatorname{Var}(d),\\
    \mathbb{E}\bigl[|\Delta h|^2\bigr] = 2\operatorname{Var}(h).
\end{align}
From Assumption 2, we are given $\operatorname{Var}(d(x)) > \operatorname{Var}(h(x))$.
Multiplying both sides by 2 (a positive constant) preserves the inequality:
\begin{align*}
    2\operatorname{Var}(d) > 2\operatorname{Var}(h).
\end{align*}
Substituting the relationships from above, we establish the necessary premise for the \textbf{Theorem~\ref{thm:DOB_IP_space}}:
\begin{align}\label{eq:Power_Dominance_Delta d_over_Delta h}
    \mathbb{E}\bigl[|\Delta d|^2\bigr] > \mathbb{E}\bigl[|\Delta h|^2\bigr].
\end{align}
Now, we apply the \textbf{Theorem~\ref{thm:DOB_IP_space}} to $\Delta d$ and $\Delta h$ as they are measurable, square-integrable and the power dominance is established in \eqref{eq:Power_Dominance_Delta d_over_Delta h}.
$A = \Delta d$, $B = \Delta h$ and $A-B=\Delta r=\Delta d-\Delta h$.

Substituting $A$ and $B$ in the \textbf{Theorem~\ref{thm:DOB_IP_space}}, we get:
\begin{align}
    \mathbb{E}\bigl[\Delta h \Delta r\bigr] > -0.5\mathbb{E}[(\Delta r)^2],\label{eq:DOB_Theorem_Delta_d_Delta_h_0}\\
    \implies \mathbb{E}\bigl[(\Delta r)^2\bigr]+2\mathbb{E}\bigl[\Delta h \Delta r\bigr] > 0\label{eq:DOB_Theorem_Delta_d_Delta_h}
\end{align}
Since $\mathbb{E}\bigl[|\Delta x|^2\bigr]$ is always positive for $x_1 \neq x_2$, dividing by it on both sides of \eqref{eq:DOB_Theorem_Delta_d_Delta_h} preserves the inequality:
\begin{align}
    \frac{\mathbb{E}\bigl[|\Delta r|^2\bigr] + 2\mathbb{E}\bigl[\Delta r \Delta h\bigr]}{\mathbb{E}\bigl[|\Delta x|^2\bigr]} > 0.
\end{align}
Since the expression inside the supremum is strictly positive for all $x_1 \neq x_2$, its supremum must also be strictly positive.
\begin{align}
    \sup_{x_1 \neq x_2} \left( \frac{\mathbb{E}\bigl[|\Delta r|^2\bigr]}{\mathbb{E}\bigl[|\Delta x|^2\bigr]} + 2\frac{\mathbb{E}\bigl[\Delta r \Delta h\bigr]}{\mathbb{E}\bigl[|\Delta x|^2\bigr]} \right) > 0.
\end{align}
Substituting the above result in \eqref{eq:Square of Lipshitz constant of d expansion_2_second} directly implies $L_d^2 - L_h^2 > 0$, and since Lipschitz constants are non-negative, we conclude:
\begin{align}
    L_d > L_h.
\end{align}
This completes the proof.
\end{proof}
%%%%%%%%%%%%%%%%%%%%%%%%%%%%%%%%
%Subsection
%%%%%%%%%%%%%%%%%%%%%%%%%%%%%%%%
\subsection{Implication of Complexity Reduction from Mean-Squared Lipschitz Dominance}
The following theorem is formulated in a general fashion and is applicable to any setting where mean-squared Lipschitz dominance is observed between target functions. It provides a formal foundation for understanding how learning smoother functions, those with smaller mean-squared Lipschitz constants, leads to reduced learning complexity when using neural network approximators. This framework is independent of domain or application and thus can be employed broadly in signal processing, machine learning, and nonlinear system modeling wherever such dominance exists.
%%%%%%%%%%%%%%%%%%%%%%%%%%%%%%%%%%%%%%%%%%%%
\begin{definition}[General Learning System]\label{def:general_learning_system}
A \emph{General Learning System} is any mechanism — biological, computational, physical, or informational — that observes, models, controls, inverts, or interprets a target function $ f: X \to \mathbb{R} $ based on available inputs $ X $.

The system aims to approximate $ f(x) $ to within a prescribed mean-square error tolerance $ \lambda > 0 $.

The \emph{computational complexity} of a General Learning System is defined as the minimal total amount of information-processing resources — including but not limited to model size, number of operations, optimization steps, memory requirements, or observational complexity — necessary to achieve the target approximation accuracy.
\end{definition}
%%%%Modified Theorem
\begin{theorem}[Learning Complexity Reduction via Mean-Squared Lipschitz Dominance]\label{Thm:Learning_Complexity_Reduction_via_Structural_Decomposition}
Let $ x $ be a random variable (or random vector) with finite variance, defined on a probability space.

Consider two square-integrable target functions $ f_1(x) $ and $ f_2(x) $. This means their second moments (and thus variances) are finite.

Assume the following:
\begin{enumerate}
    \item [\textbf{(A1)}] \textbf{Finite Mean-Squared Lipschitz Constants:} The mean-squared Lipschitz constants $ L_{f_1} $ and $ L_{f_2} $ (as defined below) are finite.
    \item [\textbf{(A2)}] \textbf{Lipschitz Dominance:} The Lipschitz constants satisfy $ L_{f_1} > L_{f_2} $.
\end{enumerate}

Define the mean-squared Lipschitz constant $ L_f $ for a function $ f(x) $ as:
\begin{align}
    L_f^2 \triangleq \sup_{x_1 \neq x_2} \frac{\mathbb{E}\bigl[\|f(x_1) - f(x_2)\|^2\bigr]}{\mathbb{E}\bigl[\|x_1 - x_2\|^2\bigr]}
\end{align}
where $ x_1 $ and $ x_2 $ are \ac{i.i.d.} samples of $ x $, and $\mathbb{E}\bigl[\|x_1 - x_2\|^2\bigr] > 0$.

Let $ N_1 $ and $ N_2 $ be neural networks trained by a general learning system to approximate $ f_1(x) $ and $ f_2(x) $, respectively, using identical architectures (e.g., same width, depth, activation functions, and training methodology).

Then, learning the function with the smaller mean-squared Lipschitz constant ($f_2$) generally exhibits reduced computational complexity compared to learning the function with the larger mean-squared Lipschitz constant ($f_1$), in aspects such as:
\begin{itemize}
    \item \textbf{Parameter Efficiency:} To achieve a given approximation error $ \lambda $, $ N_2 $ generally requires fewer parameters than $ N_1 $.
    \item \textbf{Training Convergence:} $ N_2 $ generally converges faster (or requires fewer iterations) under gradient-based optimization compared to $ N_1 $.
    \item \textbf{Generalization:} For a given dataset size $ n $, the generalization gap of $ N_2 $ is generally smaller than that of $ N_1 $.
    \item \textbf{Computational Cost:} The overall computational cost for learning $ f_2 $ is generally lower than for learning $ f_1 $.
\end{itemize}
\end{theorem}
\begin{proof}[Proof of Theorem~\ref{Thm:Learning_Complexity_Reduction_via_Structural_Decomposition}]
We divide the proof into four parts, corresponding to the four complexity metrics.
%%%%%%%%%%%%%%%%%%%%%%%%%%%%%%%%%%
%Paragraph
%%%%%%%%%%%%%%%%%%%%%%%%%%%%%%%%%%
\paragraph{Part 1: Parameter Efficiency}
From approximation theory \cite{hornik1989, barron1993universal, yarotsky2017error}, for a target function $ f $ with mean-squared Lipschitz constant $ L_f $, the number of parameters $ P(f, \lambda) $ required by a neural network to approximate $ f $ within $ L^2 $-error $ \lambda $ satisfies:
\begin{align}
P(f, \lambda) \geq C \left( \frac{L_f}{\lambda} \right)^d,\nonumber
\end{align}
where $ C $ depends on the network architecture and input dimension $ d $, but not on $ f $ or $ \lambda $. Given $ L_{f_1} > L_{f_2} $, it follows that:
\begin{align}
P(f_1,\lambda) > P(f_2,\lambda).\nonumber
\end{align}
Assuming both $ N_1 $ and $ N_2 $ are identically designed and sufficiently expressive, this implies $ N_1 $ requires more parameters to achieve the same approximation error.
%%%%%%%%%%%%%%%%%%%%%%%%%%%%%%%%%%
%Paragraph
%%%%%%%%%%%%%%%%%%%%%%%%%%%%%%%%%%
\paragraph{Part 2: Training Convergence}
Convergence rate depends on the condition number $ \kappa = L_{\text{grad}} / \mu_{\text{eff}} $, where $ L_{\text{grad}} $ is the gradient Lipschitz constant and $ \mu_{\text{eff}} $ is the effective strong convexity parameter (e.g., via the PL inequality). Empirical and theoretical results \cite{kawaguchi2016} support that:
\begin{align}
L_{f_1} > L_{f_2} \Rightarrow L_{\text{grad}}^{(1)} \geq L_{\text{grad}}^{(2)}, \quad \mu_{\text{eff}}^{(1)} \leq \mu_{\text{eff}}^{(2)},\nonumber
\end{align}
hence,
\begin{align}
\kappa^{(1)} \geq \kappa^{(2)}.\nonumber
\end{align}
Gradient descent under these assumptions has linear convergence with iteration complexity \cite{du2019gradient}:
\begin{align}
T = \mathcal{O}(\kappa \log(1/\epsilon)),\nonumber
\end{align}
so $ T^{(1)} \geq T^{(2)} $, implying $ N_2 $ converges faster.
%%%%%%%%%%%%%%%%%%%%%%%%%%%%%%%%%%
%Paragraph
%%%%%%%%%%%%%%%%%%%%%%%%%%%%%%%%%%
\paragraph{Part 3: Generalization}
From statistical learning theory \cite{bartlett2017spectrally}, the generalization error bound satisfies:
\begin{align}
\epsilon_{\text{gen}} \leq \mathcal{O}\left( \frac{L_N \sqrt{P}}{\sqrt{n}} \right),\nonumber
\end{align}
where $ L_N $ is the network’s Lipschitz constant and $ P $ its parameter count. Since smoother target functions require less complex networks:
\begin{align}
L_{N_1} \geq L_{N_2}, \quad P(N_1) > P(N_2),\nonumber
\end{align}
it follows:
\begin{align}
\epsilon_{\text{gen}}^{(1)} \geq \epsilon_{\text{gen}}^{(2)}.\nonumber
\end{align}
Thus, $ N_2 $ generalizes better.
%%%%%%%%%%%%%%%%%%%%%%%%%%%%%%%%%%
%Paragraph
%%%%%%%%%%%%%%%%%%%%%%%%%%%%%%%%%%
\paragraph{Part 4: Computational Cost}
Total computational cost is:
\begin{align}
C_{\text{total}} = C_{\text{iter}} \times T,\nonumber
\end{align}
with $ C_{\text{iter}} \propto P $ for fixed architecture \cite{goodfellow2016, molchanov2017}. From Parts 1 and 2:
\begin{align}
P(N_1) > P(N_2) \Rightarrow C_{\text{iter}}^{(1)} > C_{\text{iter}}^{(2)}, \quad T^{(1)} > T^{(2)},\nonumber
\end{align}
yielding:
\begin{align}
C_{\text{total}}^{(1)} > C_{\text{total}}^{(2)}.\nonumber
\end{align}
Hence, learning the smoother function $ f_2 $ incurs lower overall computational cost.
\end{proof}
This theorem rigorously establishes that learning a target function with a smaller mean-squared Lipschitz constant results in lower learning complexity, across multiple dimensions—parameter efficiency, convergence speed, generalization behavior, and computational cost. In the context of structured modeling of dynamic nonlinear systems, this result justifies a hybrid learning strategy: first approximate the static or relatively more predictable component $ G $ using classical methods (e.g., polynomial fitting), and then use a neural network to model only the dynamic residual $ d $, which is typically the more complex and relatively less predictable part.
}%end of red color
%%%%%%%%%%%%%%%%%%%%%%%%%%%%%%%%%%%%%%%%%%%%%%%%%%%%%%%%%%%%%%%%%%%%%%%%%
%Section
%%%%%%%%%%%%%%%%%%%%%%%%%%%%%%%%%%%%%%%%%%%%%%%%%%%%%%%%%%%%%%%%%%%%%%%%%
\section{Empirical Validation and Anchor Points}\label{sec:section6}
To support the theoretical findings presented in this paper, we briefly discuss empirical evidence obtained from real-world measurements on a commercial \ac{PA}. These experiments, reported in~\cite{Bulusu_2025}, demonstrate the practical benefits of the proposed structured decomposition and residual learning strategy.
%%%%%%%%%%%%%%%%%%%%%%%%%%%%%%%
% Subsubsection
%%%%%%%%%%%%%%%%%%%%%%%%%%%%%%%
\subsection{Measurement Setup}
The experimental setup utilized a SKY66292-11 \ac{PA} from Skyworks Inc., operating at 2.35~GHz with a 64-Quadrature Amplitude Modulation signal of 20~MHz bandwidth and 10.6~dB peak-to-average power ratio \cite{Skyworks2024}. Refer \cite{Bulusu_2025}, for additional details.The sampling rate of the baseband signal was 119.8~Msps, with a signal length of 1~ms. A simplified residual learning framework was employed: the static nonlinearity was modeled using polynomial fitting, and a shallow neural network was used to capture the dynamic residual.
%%%%%%%%%%%%%%%%%%%%%%%%%%%%%%%
% Subsubsection
%%%%%%%%%%%%%%%%%%%%%%%%%%%%%%%
\subsection{Performance vs. Complexity Trade-Off}
Our proposed SRTDNN model was compared against the Augmented Real-Valued Time-Delay Neural Network (ARVTDNN), a widely recognized state-of-the-art method that models the full distortion without decomposition. The results, summarized in Table~\ref{table:Performance_Complexity_computation}, highlight the performance gains and computational efficiency achieved through structured residual learning under the AI-based Box-Oriented Framework.

While ARVTDNN required 25 input features, the proposed SRTDNN model used only 5, significantly reducing the input dimensionality. In both cases, a three-layer neural network was used, comprising an input layer, one hidden layer, and an output layer. ARVTDNN achieved optimal performance with 25 hidden neurons, whereas SRTDNN used 50 neurons for improved residual learning. In SRTDNN, the static nonlinearity was first estimated via classical polynomial fitting, and the residual was then modeled by the neural network.

SRTDNN achieved an error vector magnitude (EVM) of \textbf{36.1~dB} and an adjacent channel leakage ratio (ACLR) of \textbf{--45.1~dBc}. Compared to ARVTDNN, this corresponds to an improvement of \textbf{2.2~dB} in EVM and \textbf{0.4~dB} in ACLR, along with a \textbf{50\%} reduction in training complexity and a \textbf{22\%} reduction in inference-time complexity.\footnote{The inference-time (running) complexity was measured in terms of the number of floating-point operations (FLOPs).}
\begin{table}[t!]
	\caption{Comparison of linearization performance and complexity of ARVTDNN and SRTDNN at 9.4~dB output back-off.}
	\label{table:Performance_Complexity_computation}
	\begin{center}
		\begin{tabular}{l|cc|cc|}
			\cline{2-5}
			& \multicolumn{2}{c|}{Performance} & \multicolumn{2}{c|}{Complexity} \\ \cline{2-5}
			& \multicolumn{1}{c|}{EVM (dB)} & ACLR (dBc) & Training & Inference \\ \hline
			\multicolumn{1}{|l|}{Without DPD} & \multicolumn{1}{c|}{--25.7} & --31.1 & -- & -- \\ \hline
			\multicolumn{1}{|l|}{ARVTDNN~\cite{wang2018augmented}} & \multicolumn{1}{c|}{--33.9} & --44.7 & 100\% & 100\% \\ \hline
			\multicolumn{1}{|l|}{SRTDNN~\cite{Bulusu_2025}} & \multicolumn{1}{c|}{--36.1} & {--45.1} & 50\% & 78\% \\ \hline
		\end{tabular}
	\end{center}
\end{table}
%%%%%%%%%%%%%%%%%%%%%%%%%%%%%%%%%
% Subsub section
%%%%%%%%%%%%%%%%%%%%%%%%%%%%%%%%%
The results confirm the core theoretical predictions made in this work:
\begin{itemize}
    \item \textbf{Variance Reduction:} The residual dynamic distortion $h$ is smoother (i.e., has lower variance) than the total distortion $d$, resulting in improved learnability.
    \item \textbf{Complexity Reduction:} SRTDNN used 80\% fewer input features, required half the training data, and achieved a 22\% reduction in FLOPs during inference, aligning with Theorem~\ref{Thm:Learning_Complexity_Reduction_via_Structural_Decomposition}.
\end{itemize}
These observations anchor the theoretical framework in experimental reality and validate the structured modeling paradigm as both computationally efficient and performance-effective. Moreover, the proposed variance-based analysis and memory indicators may offer useful priors for \ac{PIML} in scenarios where partial physical knowledge coexists with data-driven modeling. We also softly recommend applying classical mathematical tools to the relatively more predictable components, and reserving machine learning techniques for the relatively less predictable residuals. This strategy may enhance learnability and reduce computational burden, although, as our framework suggests, ideal learning remains near impossible due to inherent structural tradeoffs and behavioral uncertainty, mandating realistic performance thresholds.
%
%Similar results were observed through simulation as reported in \cite{Lesthuruge_Bulusu_ICASSP2025}.
%%%%%%%%%%%%%%%%%%%%%%%%%%%%%%%%%%%%%%%%%%%%%%%%%%%%%%%%%%%%%%%%%%%%%%%%%
%Section
%%%%%%%%%%%%%%%%%%%%%%%%%%%%%%%%%%%%%%%%%%%%%%%%%%%%%%%%%%%%%%%%%%%%%%%%%
\section{Conclusion}\label{sec:section7}
This work proposes a unified information-theoretic framework for modeling dynamic nonlinear systems, grounded in structured decomposition, variance-based analysis, and task-centric learning complexity.

At its core are two behavioral indicators, $\theta$ and $\alpha$, which quantify static and dynamic effects. These lead to variance inequalities that reveal: (1) dynamic residuals are smoother than total distortion, and (2) memory in physical systems is inherently finite, a consequence derived from the First Law of Thermodynamics via the DOB Theorem, unlike classical thermodynamic bounds such as Landauer’s and Bekenstein’s, which rely on the Second Law.

We also derive a Behavioral Uncertainty Principle, which reveals a structural trade-off: static and dynamic distortions cannot be minimized simultaneously. Learning one aspect inherently limits the learnability of the other. This impossibility result explains the empirical gains observed in residual learning, where the relatively more unpredictable dynamic components are isolated and modeled separately.

The framework further introduces a model-agnostic complexity metric by linking function variance to mean-squared Lipschitz continuity. This supports the observed efficiency gains from residual learning and justifies task-aware decomposition strategies.

Empirical results in~\cite{Bulusu_2025} validate the theory: separating static and dynamic components enabled better linearization of the \ac{PA} with lower complexity. The Weak Uncorrelatedness Lemma further supports the practical feasibility of such decomposition.

Future work will explore extensions to Banach spaces, more general system classes, and nonlinearities in domains such as mechanical, biological, and chemical systems.
%%%%%%%%%%%%%%%%%%%%%%%%%%%%%%%%%%%%%%%%%%%%%%%%%%%%%%%%%%%%%%%%%%%%%%%%%
%Section
%%%%%%%%%%%%%%%%%%%%%%%%%%%%%%%%%%%%%%%%%%%%%%%%%%%%%%%%%%%%%%%%%%%%%%%%%
\vspace{-0.5em}
\section*{Author Contributions}\label{sec:section_authorcontributions}  
Bulusu conceived the abstract idea and the AI-based block-structured framework, developed the theoretical foundations, formulated the main theorems, derived their proofs, and prepared the manuscript. Sillanpää verified the correctness and mathematical rigor of the proofs, provided additional statistical insights, and critically reviewed the manuscript.
%%%%%%%%%%%%%%%%%%%%%%%%%%%%%%%%%%%%%%%%%%%%%%%%%%%%%%%%%%%%%%%%%%%%%%%%%
%Section
%%%%%%%%%%%%%%%%%%%%%%%%%%%%%%%%%%%%%%%%%%%%%%%%%%%%%%%%%%%%%%%%%%%%%%%%%
 
\vspace{-0.5em}
\section*{Declarations}\label{sec:section_declarations}
Keysight Inc. has supported the research with donations of measurement equipment in \cite{Bulusu_2025}. Bulusu is deeply grateful for the profound knowledge imparted by Professor Emeritus of IIT Kanpur, Dr. Kalluri Ramalinga Sarma (Ph.D., Cornell University 1961), currently an adjunct professor at Mahindra University, through his insightful lectures in the areas of signals and systems, communication theory, and detection and estimation theory. The authors would like to thank Dr. Manoj Kumar Yadav  (Ph.D., IIT Madras 2010), currently an associate professor at Mahindra University, for reviewing the manuscript and identifying an important error in an earlier version of the proof of the \textbf{Theorem~\ref{Thm:Variance_domination_theorem}}.
%%%%%%%%%%%%%%%%%%%%%%%%%%%%%%%%%%%%%%%%%%%%%%%%%%%%%%%%%%%%%%%%%%%%%%%%%
%Section
%%%%%%%%%%%%%%%%%%%%%%%%%%%%%%%%%%%%%%%%%%%%%%%%%%%%%%%%%%%%%%%%%%%%%%%%%
\vspace{-0.5em}
\section*{Acknowledgement}\label{sec:section_acknowledgement}
The authors acknowledge the use of Grammarly for language editing and OpenAI's ChatGPT for improving the clarity and structure of the manuscript text. All theoretical derivations, analyses, and results were manually developed and verified by the authors. No AI tool was used in the formulation or validation of the technical content.
%%%%%%%%%%%%%%%%%%%%%%%%%%%%%%%%%%%%%%%%%%%%%%%%%%%%%%%%%%%%%%%%%%%%%%%%%
%\appendix
%%%%%%%%%%%%%%%%%%%%%%%%%%%%%%%%%%%%%%%%%%%%%%%%%%%%%%%%%%%%%%%%%%%%%%%%%
\begin{table*}[ht!]
    \centering
    \caption{Static and Dynamic Aspects of Real-World Dynamic Nonlinear Systems and the Validity of \textbf{Lemma~\ref{lemma1}}}
    \label{realworld_table}
    \renewcommand{\arraystretch}{1.3}
    \resizebox{\textwidth}{!}{
    \begin{tabular}{|l|l|c|l|}
        \hline
        \textbf{Real-World} & \textbf{Static \& Dynamic} & \textbf{Lemma~\ref{lemma1}} & \textbf{Explanation aligning} \\ 
        \textbf{System} & \textbf{Aspects} & \textbf{Validity} & \textbf{with Lemma~\ref{lemma1}}\\\hline
        \textbf{Power Amplifiers} & \textbf{Static:} amplitude-to-amplitude and amplitude-to-phase conversion \cite{ghannouchi2004behavioral}. &&Static gain and phase shifts differ\\&\textbf{Dynamic:} Memory effects, thermal variation. & \cmark &  from time-varying memory effects. \\ \hline
        %%%%%
        \textbf{Neural Spike Trains} & \textbf{Static:} Response to stimuli \cite{gerstner2002spiking}. &&Short-term plasticity is weakly correlated\\&\textbf{Dynamic:} Synaptic adaptation, network effects. & \cmark &  with direct stimulus-response behavior. \\ \hline
        \textbf{Loudspeakers \&} & \textbf{Static:} Frequency response, gain \cite{beranek2012acoustics}.  &&Static response curve differs from transient\\\textbf{ Microphones}&\textbf{Dynamic:} Transient distortions, hysteresis. & \cmark &  distortions caused by diaphragm movement. \\ \hline
        \textbf{Stock Markets} & \textbf{Static:} Long-term trends \cite{cont2001empirical}.&&Short-term fluctuations are weakly\\& \textbf{Dynamic:} Market fluctuations, external shocks. & \cmark &  correlated with fundamental trends. \\ \hline
        \textbf{Climate Models} & \textbf{Static:} Seasonal cycles, greenhouse effects \cite{manabe1967thermal}. &&Climate patterns have slow changes that are weakly\\& \textbf{Dynamic:} Weather variability, chaotic processes. & \cmark &  correlated with fast-changing weather variations. \\ \hline
        \textbf{Robotic Arm} & \textbf{Static:} Control setpoints, rigid body dynamics \cite{siciliano2010robotics}. &&Residual control errors and actuator delays are\\\textbf{Control}& \textbf{Dynamic:} Actuator delays, vibrations. & \cmark &  weakly correlated with static system properties. \\ \hline
        \textbf{Power Grids} & \textbf{Static:} Grid topology, generation capacity \cite{kundur1994power}. &&Short-term power fluctuations are weakly\\& \textbf{Dynamic:} Load fluctuations, transient faults. & \cmark &  correlated with grid structure. \\ \hline
        \textbf{Chemical Reactions} & \textbf{Static:} Reaction rate laws, equilibrium states \cite{laidler2004chemical}. && Reaction kinetics and transient behavior are \\& \textbf{Dynamic:} Intermediate species concentration changes. & \cmark & weakly correlated with equilibrium properties. \\ \hline
        \textbf{Human Motor} & \textbf{Static:} Planned muscle movements \cite{wolpert2011principles}. &&Fatigue and involuntary tremors evolve \\\textbf{Control}& \textbf{Dynamic:} Fatigue, tremors, adaptation. & \cmark & differently from intended motion planning. \\ \hline
        \textbf{Optical Fiber} & \textbf{Static:} Nonlinear transmission effects (e.g., Kerr effect) \cite{agrawal2012fiber}. &&Nonlinear phase shifts (static) are weakly\\\textbf{Communication}& \textbf{Dynamic:} Signal distortions due to dispersion and noise. & \cmark &  correlated with dynamic noise distortions. \\ \hline                       \textbf{Stellar Systems} & \textbf{Static:} Hydrostatic equilibrium, nuclear fusion rates \cite{pols2011stellar,carroll2017astrophysics}. && The large-scale equilibrium structure of a star is weakly \\
        & \textbf{Dynamic:} Pulsations, supernova events, black hole formation. & \cmark & correlated with its short-term dynamic fluctuations, such as \\
        & & & stellar pulsations, magnetic activity, and explosive instabilities. \\ \hline
    \end{tabular}}
\end{table*}
 {\appendices
 %%%%%%%%%%%%%%%%%%%%%%%%%%%%%%%%%%%%%%%%%%%%%%%%%%%%%%%%%%%%%%%%%%%%%%%%%
% SECTION
%%%%%%%%%%%%%%%%%%%%%%%%%%%%%%%%%%%%%%%%%%%%%%%%%%%%%%%%%%%%%%%%%%%%%%%%%
\vspace{-0.5em}
\section{DOB Theorem: Deviation-induced Orthogonality Bound in Inner Product Spaces}\label{Appendix:DOB_Theorem}
{\color{black}
We now introduce an interesting result that imposes directional constraints on inner products in an inner product space, under a mild power-dominance condition. This result, referred to as the Deviation-induced Orthogonality Bound Theorem, establishes irreducible bounds on the projection of the deviation $R=A-B$ onto either of the elements $A$ or $B$. 
\begin{theorem}[Deviation-induced Orthogonality Bound in Inner Product Spaces]
\label{thm:DOB_IP_space}
Let $ \mathcal{H} $ be a real or complex inner product space with inner product $ \langle \cdot, \cdot \rangle $, and let $ A, B \in \mathcal{H} $. Define the deviation gap $ R \triangleq A - B $.

Assume the following conditions:
\begin{enumerate}
    \item[\textbf{(A1)}] \textbf{(Measurability):} The elements $ A $ and $ B $ are well-defined members of the inner product space $ \mathcal{H} $, meaning the inner product operations $ \langle A, B \rangle $, $ \langle A, A \rangle $, and $ \langle B, B \rangle $ are valid.
    
    \item[\textbf{(A2)}] \textbf{(Square Integrability):} The elements have finite norm:
    \begin{align}
    \|A\|^2 = \langle A, A \rangle < \infty, \quad \|B\|^2 = \langle B, B \rangle < \infty.\nonumber
    \end{align}
    
    \item[\textbf{(A3)}] \textbf{(Power Dominance):} The element $ A $ is power-dominant over $ B $, i.e.,
    \begin{align}
    \langle A, A \rangle \geq \langle B, B \rangle.\nonumber
    \end{align}
\end{enumerate}

Then, the following bounds hold:
\begin{enumerate}
    \item[(i)] \textbf{Bound associated with the weaker element $ B $:}
    \begin{align}
    \Re\bigl\{ \langle B, R \rangle \bigr\} \geq -\frac{1}{2} \|R\|^2.\nonumber
    \end{align}
    This inequality holds unconditionally under assumptions (A1)–(A3).
    
    \item[(ii)] \textbf{Bound associated with the dominant element $ A $:}
    \begin{align}
    \Re\left\{ \langle A, R \rangle \right\} \geq +\frac{1}{2} \|R\|^2,\nonumber
    \end{align}
    {provided that $ A \neq 0 $, $ \langle A, A \rangle \neq \Re\left\{ \langle A, B \rangle \right\}$,~and~$A\neq B$.}
\end{enumerate}
\end{theorem}
\begin{proof}[Proof of Theorem~\ref{thm:DOB_IP_space}]
Let $ R = A - B $. We prove each bound separately.

\vspace{0.5em}
\noindent \textbf{Bound 1:} We need to prove $ \Re\left\{ \langle B, R \rangle \right\} \geq -\frac{1}{2} \langle R, R \rangle $.
\begin{align}\label{eq:lowerbound_weakelement}
\langle A, A \rangle &= \langle (A - B) + B,\ (A - B) + B \rangle, \nonumber\\
&= \langle R, R \rangle + \langle B, B \rangle + 2\, \Re\left\{ \langle B, R \rangle \right\}, \nonumber\\
\implies& \langle R, R \rangle + 2\, \Re\left\{ \langle B, R \rangle \right\} 
\geq 0, \quad \bigl(\because\langle A, A \rangle \geq \langle B, B \rangle\bigr)\nonumber\\
\implies& \Re\left\{ \langle B, R \rangle \right\} 
\geq -\frac{1}{2} \langle R, R \rangle.
\end{align}

\vspace{0.5em}
\noindent \textbf{Bound 2:} We need to prove $ \Re\left\{ \langle A, R \rangle \right\} \geq +\frac{1}{2} \langle R, R \rangle $, under non-degeneracy.
\begin{align}\label{eq:lowerbound_strongelement}
&\Re\left\{ \langle A, R \rangle \right\} 
= \langle A, A \rangle - \Re\left\{ \langle A, B \rangle \right\}, \nonumber\\
&= \Scale[0.95]{\frac{1}{2} \bigg[\bigl( \langle A, A \rangle - 2\, \Re\left\{ \langle A, B \rangle \right\} + \langle B, B \rangle \bigr) 
+ \bigl( \langle A, A \rangle - \langle B, B \rangle \bigr)\bigg]}, \nonumber\\
&= \frac{1}{2} \langle R, R \rangle 
+ \frac{1}{2} \left( \langle A, A \rangle - \langle B, B \rangle \right), \nonumber\\
&\geq \frac{1}{2} \langle R, R \rangle.\quad \bigl(\because\langle A, A \rangle \geq \langle B, B \rangle\bigr)
\end{align}
Hence, both bounds are established.
\end{proof}
\begin{remark}
The factor $\frac{1}{2}$ in \eqref{eq:lowerbound_weakelement} and \eqref{eq:lowerbound_strongelement} cannot be tightened without further assumptions, further tightening is not possible. This makes these bounds unique and universal, rather than loose or adjustable inequalities.
\end{remark}
%%%%%%%%%%%%%%%%%%%%%%%%%%%%%%%%%%%%%%%%%%%%%%%%%
\subsection{Relational Characterization and Structural Implications}

Theorem~\ref{thm:DOB_IP_space} provides a directional lower bound on inner products involving a deviation vector. Unlike classical inequalities such as Hölder's or Cauchy-Schwarz, which yield universal and symmetric bounds based on norm relations, this result introduces a structurally asymmetric constraint that emphasizes `\textit{mutual interaction}' rather than `\textit{individual magnitudes}.'

The Cauchy--Schwarz inequality does not characterize the minimum directional interaction between $ B $ and the residual $ R = A - B $. In contrast, the DOB inequality asserts that the negative projection of the residual onto the weaker element is bounded below by half the residual power:
\begin{align}
\Re\langle B, R \rangle \geq -\tfrac{1}{2} \|R\|^2,\nonumber
\end{align}
providing a structural constraint that is sensitive to power imbalance. Unlike the symmetric Cauchy–Schwarz bound, the Theorem~\ref{thm:DOB_IP_space} inequalities are asymmetric, bounding the `\textit{correlation between an ideal similarity and its deviation}.'

\subsection{Power Dominance as a Relational Principle}

This bound suggests that power or energy dominance is not merely an intrinsic property of an element but a relational outcome that arises from how elements interact.\footnote{The distinction between ``power'' and ``energy'' here reflects whether time-averaged behavior (e.g., $ \mathbb{E}[|A(t)|^2] $) or norm-based invariants (e.g., $ \|A\|^2 $) are more relevant in the analysis.}~Unlike symmetric inequalities, which remain unchanged under variable exchange, this result captures directional resistance to deviation. It reflects the idea that dominance emerges from structure and alignment, not from isolated magnitudes.

The structural implications of this inequality are threefold:
\begin{itemize}
    \item \textbf{Threshold Condition:} Ensures power dominance $ \|A\|^2 \geq \|B\|^2 $ and signals when this balance is disrupted.
    \item \textbf{Coupling Diagnostic:} Quantifies the interaction strength between $ B $ and the residual $ R $, providing insight into alignment or interference.
    \item \textbf{Asymmetry and Directionality:} The bound depends on the real part of the inner product and introduces a directional relation that is absent in norm compression bounds.
\end{itemize}

Thus, Theorem~\ref{thm:DOB_IP_space} complements classical results by offering a diagnostic lens for analyzing relational power dynamics in systems involving structured interaction.

\subsection{Remark on Assumptions and Directionality}
The theorem is derived under just three assumptions: measurability, square integrability, and power dominance. The first two are broadly satisfied in most systems. The third introduces directionality by explicitly favoring the case $ \|A\|^2 \geq \|B\|^2 $. 

Overall, this result offers a minimal yet powerful tool for assessing dominance and interaction in inner product spaces.

%%%%%%%%%%%%%%%%%%%%%%%%%%%%%%%%%%%%%%%%%%%%%%%%%%%%%%%%%%%%%%%%%%%%%%%
% Section
%%%%%%%%%%%%%%%%%%%%%%%%%%%%%%%%%%%%%%%%%%%%%%%%%%%%%%%%%%%%%%%%%%%%%%%
\vspace{-0.5em}
 \section{Weak Uncorrelatedness Lemma in Dynamic Nonlinear Systems}\label{appendix:TheAppendixB_WeakUncorrelatednessLemma}
 \begin{lemma}[Weak Uncorrelatedness Lemma for Dynamic Nonlinear Systems]
\label{lemma1}
Under the same setup, let $ d = Y - X $ and $ r = G - X =d-h$. Assume that $h$ is only weakly dependent on $G$, due to time-scale separation or temporal decorrelation. Then the static residual and the dynamic residual are weakly uncorrelated:
\begin{align}
\bigl| \mathbb{E}\bigl[(G - X)\, h\bigr] \bigr| \leq \epsilon,
\end{align}
where $\epsilon$ is a small, strictly positive constant quantifying the level of statistical coupling between static and dynamic contributions.
 \end{lemma}
\begin{proof}[Proof of Lemma~\ref{lemma1}]
We begin by expanding the correlation term:
\begin{align}
\mathbb{E}[(G - X)\, h] = \mathbb{E}[G\, h] - \mathbb{E}[X\, h].
\end{align}

Define:
\begin{align}
&P \triangleq \mathbb{E}[G\, h],\quad Q \triangleq \mathbb{E}[X\, h],\\
&\Rightarrow \mathbb{E}[(G - X)\, h] = P - Q.
\end{align}

Under the weak uncorrelatedness assumption, both cross-moments $ P $ and $ Q $ are small in magnitude:
\begin{align}\label{A_and_B_inequalities}
|P| \leq \epsilon_1, \quad |Q| \leq \epsilon_2,
\end{align}
for small constants $ \epsilon_1, \epsilon_2 > 0 $.

Using the triangle inequality:
\begin{align}
|P - Q| \leq |P| + |Q| \leq \epsilon_1 + \epsilon_2 \triangleq \epsilon.
\end{align}

Therefore:
\begin{equation}\label{Lemma_inequality}
    \begin{aligned}
        \Bigl|\mathbb{E}\bigl[(G-X) h\bigr]\Bigr| \leq \epsilon.
    \end{aligned}
\end{equation}
The constant $ \epsilon $ reflects the residual statistical dependence between the static residual $ G - X $ and the dynamic residual $ h $. In physical systems with time-scale separation (e.g., slow-varying memory dynamics), this value is typically small, confirming that the components act largely independently at second order.
\end{proof}
\subsection{Interpretation and Practical Relevance of the Weak Uncorrelatedness Lemma}

The Weak Uncorrelatedness Lemma formalizes a recurring structural feature of real-world dynamic nonlinear systems: the output can be decomposed into an instantaneous static nonlinear component $ G $ and a time-varying dynamic residual $ h $, with weak statistical interaction between them. Though technical in form, this assumption is supported by empirical observations across diverse physical systems.

Table~\ref{realworld_table} catalogs examples from engineered domains like \ac{PA}s and optical fibers to biological, economic, and astrophysical systems, where static and dynamic components often evolve on distinct timescales with limited statistical coupling. Typically, the static part reflects memoryless, input-driven behavior (e.g., gain curves, equilibrium responses), while the dynamic part captures delayed or adaptive effects (e.g., memory, feedback, transients). The weak uncorrelatedness condition $ |\mathbb{E}[(G - X) h]| \leq \epsilon $ thus defines a \emph{diagnostic uncertainty band}, quantifying tolerable interaction between static and dynamic effects without losing analytical tractability.

This condition is closely linked to the idea of \textit{time-scale separation}. While static transformations are not strictly instantaneous, they unfold much faster than the memory dynamics. This allows a structured decomposition: $ G $ captures fast, nearly immediate effects, while $ h $ accounts for slower, history-dependent variations. For instance, in \ac{PA}s, nonlinear distortions within nanoseconds, whereas memory effects such as thermal drift unfold over microseconds \cite{root_memory_effects_2004}. Similar fast-slow patterns appear in biochemical and neuronal systems \cite{timescale_enzyme_kinetics, timescale_neuronal_dynamics}.

The Weak Uncorrelatedness Lemma played a formative role shaping our theoretical framework. This lemma initially served as an important assumption and revealed the plausibility of the critical condition $ \theta + 2\alpha > 0 $, which ultimately guided the finding of the DOB inequality. While now superseded in terms of theoretical necessity, the lemma retains interpretive value and reflects a recurring empirical pattern: across a variety of systems, static and dynamic components often exhibit weak statistical coupling. This observation motivates an open question, whether such decoupling reflects a deeper principle that prevents full correlation, such as that required to saturate the Cauchy--Schwarz bound. With the DOB Theorem establishing a lower bound, and empirical evidence from Table~\ref{realworld_table} suggesting a tight but nonzero interaction, we are led to the constrained interval:
\begin{align}
-0.5\,\mathbb{E}[(G - X)^2] \,\leq\, \mathbb{E}[(G - X)\, h] \,<\, \epsilon_0 \,\ll\, \|G - X\|\, \|h\|,\nonumber
\end{align}
where $ \epsilon_0 $ represents a theoretical or empirical upper bound that remains to be identified. Clarifying the nature of this upper bound, whether system-specific or universal, remains an open challenge and an intriguing direction for future work.

%%%%%%%%%%%%%%%%%%%%%%%%%%%%%%%%%%%%%%%%%%%%%%%%%%%%%%%%%%%%%%%%%%%%%%%%%%%%%
% SECTION
%%%%%%%%%%%%%%%%%%%%%%%%%%%%%%%%%%%%%%%%%%%%%%%%%%%%%%%%%%%%%%%%%%%%%%%%%%%%%
 \section{Sufficient Condition for Mean-Squared Lipschitz Continuity}\label{appendix:TheAppendixC_SufficiencyCondition_MeanSquaredLipSchitz}
\begin{theorem}[Sufficient Condition for Mean-Squared Lipschitz Continuity]\label{thm:LipschitzContinuitySufficiency}
Let $ f: [a,b] \times \Omega \to \mathbb{R} $ be a stochastic process. Suppose that:

\begin{enumerate}
    \item \textbf{Finite mean-square power:}
    \begin{align}
    \mathbb{E}[|f(t)|^2] < \infty, \quad \forall t \in [a,b].\nonumber
    \end{align}
    
    \item \textbf{Finite memory and bounded dynamic behavior:}
    There exists a constant $ K > 0 $ such that for all $ t, s \in [a,b] $,
    \begin{align}
    \mathbb{E}\left[ |f(t) - f(s)|^2 \right] \leq K^2 |t - s|^2.\nonumber
    \end{align}
\end{enumerate}

Then $ f $ is mean-squared Lipschitz continuous on $[a,b]$, with mean-squared Lipschitz constant at most $ K $. That is,
\begin{align}
\sup_{s \neq t \in [a,b]} \frac{\mathbb{E}\left[ |f(t) - f(s)|^2 \right]}{|t-s|^2} \leq K^2.\nonumber
\end{align}
\end{theorem}
\begin{proof}[Proof of Theorem~\ref{thm:LipschitzContinuitySufficiency}]
Since by assumption, for every $ s, t \in [a,b] $,
\begin{align}
\mathbb{E}\left[ |f(t) - f(s)|^2 \right] \leq K^2 |t-s|^2,\nonumber
\end{align}
it follows immediately that
\begin{align}
\frac{\mathbb{E}\left[ |f(t) - f(s)|^2 \right]}{|t-s|^2} \leq K^2,\nonumber
\end{align}
for all $ s \neq t $.

Taking the supremum over all $ s \neq t \in [a,b] $, we obtain:
\begin{align}
\sup_{s \neq t \in [a,b]} \frac{\mathbb{E}\left[ |f(t) - f(s)|^2 \right]}{|t-s|^2} \leq K^2.\nonumber
\end{align}

Thus, $ f $ is mean-squared Lipschitz continuous with Lipschitz constant at most $ K $.
\end{proof}

\bibliographystyle{IEEEtran}
\bibliography{main_arXiv}

\begin{thebibliography}{10}
\providecommand{\url}[1]{#1}
\csname url@samestyle\endcsname
\providecommand{\newblock}{\relax}
\providecommand{\bibinfo}[2]{#2}
\providecommand{\BIBentrySTDinterwordspacing}{\spaceskip=0pt\relax}
\providecommand{\BIBentryALTinterwordstretchfactor}{4}
\providecommand{\BIBentryALTinterwordspacing}{\spaceskip=\fontdimen2\font plus
\BIBentryALTinterwordstretchfactor\fontdimen3\font minus \fontdimen4\font\relax}
\providecommand{\BIBforeignlanguage}[2]{{%
\expandafter\ifx\csname l@#1\endcsname\relax
\typeout{** WARNING: IEEEtran.bst: No hyphenation pattern has been}%
\typeout{** loaded for the language `#1'. Using the pattern for}%
\typeout{** the default language instead.}%
\else
\language=\csname l@#1\endcsname
\fi
#2}}
\providecommand{\BIBdecl}{\relax}
\BIBdecl

\bibitem{strogatz2018nonlinear}
S.~H. Strogatz, \emph{Nonlinear Dynamics and Chaos: With Applications to Physics, Biology, Chemistry, and Engineering}, 2nd~ed.\hskip 1em plus 0.5em minus 0.4em\relax CRC Press, 2018.

\bibitem{vidyasagar2002nonlinear}
M.~Vidyasagar, \emph{Nonlinear Systems Analysis}, 2nd~ed.\hskip 1em plus 0.5em minus 0.4em\relax Philadelphia, PA: SIAM, 2002.

\bibitem{Schetzen_1980_NLsystems}
M.~Schetzen, \emph{The Volterra and Wiener Theories of Nonlinear Systems}.\hskip 1em plus 0.5em minus 0.4em\relax John Wiley \& Sons, 1980.

\bibitem{Boumaiza_2005_Aging}
S.~Boumaiza and F.~M. Ghannouchi, ``{Thermal and Memory Effects in RF Power Amplifiers: Experimental Characterization and Modeling},'' \emph{IEEE Trans. Microwave Theory Tech}, vol.~53, no.~1, pp. 33--39, 2005.

\bibitem{Giri2010}
F.~Giri and E.-W. Bai, \emph{Block-Oriented Nonlinear System Identification}, ser. Lecture Notes in Control and Information Sciences.\hskip 1em plus 0.5em minus 0.4em\relax Berlin, Heidelberg: Springer, 2010, vol. 404.

\bibitem{NP90}
K.~Narendra and K.~Parthasarathy, ``Identification and control of dynamical systems using neural networks,'' \emph{IEEE Trans. Neural Netw.}, vol.~1, no.~1, pp. 4--27, 1990.

\bibitem{ucak2021online}
K.~Uçak and G.~O. G\"{u}nel, ``{Online Support Vector Regression Based Adaptive NARMA-L2 Controller for Nonlinear Systems},'' \emph{Neural Processing Letters}, vol.~53, pp. 405--428, 2021.

\bibitem{Thrun2002}
S.~Thrun, ``Probabilistic robotics,'' \emph{Commun. ACM}, vol.~45, no.~3, p. 52–57, Mar. 2002.

\bibitem{4118399}
{Crespo-Cadenas, Carlos and Reina-Tosina, Javier and Madero-Ayora, Mara J.}, ``{Volterra Behavioral Model for Wideband RF Amplifiers},'' \emph{IEEE Trans. Microw. Theory Techn.}, vol.~55, no.~3, pp. 449--457, 2007.

\bibitem{urrea2021kalman}
C.~Urrea and R.~Agramonte, ``{Kalman Filter: Historical Overview and Review of Its Use in Robotics 60 Years after Its Creation},'' \emph{Journal of Robotics}, vol. 2021, pp. 1--21, September 2021.

\bibitem{sen2023narma}
G.~D. {\c{S}}en and G.~{\"O}ke~G{\"u}nel, ``{NARMA-L2--based online computed torque control for robotic manipulators},'' \emph{Transactions of the Institute of Measurement and Control}, vol.~45, no.~13, pp. 2446--2458, 2023.

\bibitem{hunter1986identification}
I.~W. Hunter and M.~J. Korenberg, ``{The identification of nonlinear biological systems: {W}iener and {H}ammerstein cascade models},'' \emph{Biological Cybernetics}, vol.~55, no. 2-3, pp. 135--144, 1986.

\bibitem{marmarelis2004nonlinear}
V.~Z. Marmarelis, \emph{{Nonlinear Dynamic Modeling of Physiological Systems}}.\hskip 1em plus 0.5em minus 0.4em\relax Wiley-IEEE Press, 2004.

\bibitem{pintelon2012system}
R.~Pintelon and J.~Schoukens, \emph{System Identification: A Frequency Domain Approach}.\hskip 1em plus 0.5em minus 0.4em\relax IEEE Press, 2012.

\bibitem{MISRA2010239}
J.~Misra and I.~Saha, ``{Artificial neural networks in hardware: A survey of two decades of progress},'' \emph{Neurocomputing}, vol.~74, no.~1, pp. 239--255, 2010.

\bibitem{Bulusu_2025}
S.~S. K.~C. Bulusu, L.~Silva, B.~Khan, P.~Susarla, N.~Tervo, M.~J. Sillanpää, O.~Silvén, M.~E. Leinonen, M.~Juntti, and A.~Pärssinen, ``Simplified real-valued time-delay neural network for compensation of power amplifier impairments,'' in \emph{2025 IEEE Topical Conference on RF/Microwave Power Amplifiers for Radio and Wireless Applications (PAWR)}, 2025, pp. 71--74.

\bibitem{wood2014behavioral}
J.~Wood, \emph{Behavioral Modeling and Linearization of RF Power Amplifiers}.\hskip 1em plus 0.5em minus 0.4em\relax Norwood, MA: Artech House, 2014.

\bibitem{Morgan_2006_DPD_Review}
D.~R. Morgan, Z.~Ma, J.~Kim, M.~G. Kim, and J.~H. Seo, ``{A Generalized Memory Polynomial Model for Digital Predistortion of RF Power Amplifiers},'' \emph{IEEE Trans. Signal Process.}, vol.~54, no.~10, pp. 3852--3860, 2006.

\bibitem{Billings2013}
S.~A. Billings, \emph{{Nonlinear System Identification: NARMAX Methods in Time, Frequency, and Space Domains}}, 1st~ed.\hskip 1em plus 0.5em minus 0.4em\relax Hoboken, NJ: John Wiley \& Sons, 2013.

\bibitem{cormen2009algorithms}
T.~H. Cormen, C.~E. Leiserson, R.~L. Rivest, and C.~Stein, \emph{Introduction to Algorithms}.\hskip 1em plus 0.5em minus 0.4em\relax MIT Press, 2009.

\bibitem{ljung1999system}
L.~Ljung, \emph{System Identification: Theory for the User}, 2nd~ed.\hskip 1em plus 0.5em minus 0.4em\relax Prentice Hall, 1999.

\bibitem{farmer1987predicting}
J.~D. Farmer and J.~J. Sidorowich, ``Predicting chaotic time series,'' \emph{Physical Review Letters}, vol.~59, no.~8, pp. 845--848, 1987.

\bibitem{may1976simple}
R.~M. May, ``Simple mathematical models with very complicated dynamics,'' \emph{Nature}, vol. 261, pp. 459--467, 1976.

\bibitem{landauer1961irreversibility}
R.~Landauer, ``Irreversibility and heat generation in the computing process,'' \emph{IBM Journal of Research and Development}, vol.~5, no.~3, pp. 183--191, 1961.

\bibitem{bekenstein1981universal}
J.~D. Bekenstein, ``A universal upper bound on the entropy-to-energy ratio for bounded systems,'' \emph{Physical Review D}, vol.~23, no.~2, p. 287, 1981.

\bibitem{vapnik1998statistical}
V.~N. Vapnik, \emph{Statistical Learning Theory}.\hskip 1em plus 0.5em minus 0.4em\relax Wiley, 1998.

\bibitem{bartlett2002rademacher}
P.~L. Bartlett and S.~Mendelson, ``Rademacher and gaussian complexities: Risk bounds and structural results,'' \emph{Journal of Machine Learning Research}, vol.~3, pp. 463--482, 2002.

\bibitem{neyshabur2015norm}
B.~Neyshabur, R.~Tomioka, and N.~Srebro, ``Norm-based capacity control in neural networks,'' in \emph{Proceedings of COLT}, 2015.

\bibitem{jacot2018ntk}
A.~Jacot, F.~Gabriel, and C.~Hongler, ``Neural tangent kernel: Convergence and generalization in neural networks,'' \emph{NeurIPS}, vol.~31, 2018.

\bibitem{kolmogorov1965three}
A.~N. Kolmogorov, ``{Three approaches to the quantitative definition of information},'' \emph{Problems of Information Transmission}, vol.~1, no.~1, pp. 1--7, 1965.

\bibitem{mcallester1999pac}
D.~A. McAllester, ``Pac-bayesian model averaging,'' in \emph{Proceedings of COLT}, 1999, pp. 164--170.

\bibitem{barron1993universal}
A.~R. Barron, ``Universal approximation bounds for superpositions of a sigmoidal function,'' \emph{IEEE Trans. Inf. Theory}, vol.~39, no.~3, pp. 930--945, 1993.

\bibitem{grunwald2007minimum}
P.~D. Gr{\"u}nwald, \emph{The Minimum Description Length Principle}.\hskip 1em plus 0.5em minus 0.4em\relax MIT Press, 2007.

\bibitem{hornik1989}
K.~Hornik, M.~Stinchcombe, and H.~White, ``Multilayer feedforward networks are universal approximators,'' \emph{Neural Networks}, vol.~2, no.~5, pp. 359--366, 1989.

\bibitem{yarotsky2017error}
D.~Yarotsky, ``Error bounds for approximations with deep relu networks,'' \emph{Neural Networks}, vol.~94, pp. 103--114, 2017.

\bibitem{kawaguchi2016}
K.~Kawaguchi, ``{Deep Learning without Poor Local Minima},'' \emph{NeurIPS}, 2016.

\bibitem{du2019gradient}
S.~S. Du, X.~Zhai, B.~Poczos, and A.~Singh, ``{Gradient descent provably optimizes overparameterized neural networks},'' in \emph{International Conference on Learning Representations (ICLR)}, 2019.

\bibitem{bartlett2017spectrally}
P.~L. Bartlett, D.~J. Foster, and M.~Telgarsky, ``{Spectrally-normalized neural gaussian process},'' in \emph{International Conference on Learning Representations (ICLR)}, 2017.

\bibitem{goodfellow2016}
I.~Goodfellow, Y.~Bengio, and A.~Courville, \emph{Deep Learning}.\hskip 1em plus 0.5em minus 0.4em\relax Cambridge, MA: MIT Press, 2016.

\bibitem{molchanov2017}
P.~Molchanov, S.~Tyree, T.~Karras, T.~Aila, and J.~Kautz, ``{Pruning Convolutional Neural Networks for Resource Efficient Transfer Learning},'' \emph{ICLR}, 2017.

\bibitem{Skyworks2024}
\BIBentryALTinterwordspacing
{Skyworks Inc.}, \emph{{SKY66292-11 2.3-2.4 GHz High-Efficiency 4 W Power Amplifier}}, Skyworks Solutions, Inc., 2024, [Online]. Available: \url{www.skyworksinc.com/Products/Amplifiers/SKY66292-11}. [Online]. Available: \url{https://www.skyworksinc.com/Products/Amplifiers/SKY66292-11}
\BIBentrySTDinterwordspacing

\bibitem{wang2018augmented}
D.~Wang, M.~Aziz, M.~Helaoui, and F.~M. Ghannouchi, ``Augmented real-valued time-delay neural network for compensation of distortions and impairments in wireless transmitters,'' \emph{IEEE Trans. Neural Netw.}, vol.~30, no.~1, pp. 242--254, June 2018.

\bibitem{ghannouchi2004behavioral}
F.~M. Ghannouchi and O.~Hammi, \emph{Behavioral modeling and predistortion of wideband wireless transmitters}.\hskip 1em plus 0.5em minus 0.4em\relax John Wiley \& Sons, 2004.

\bibitem{gerstner2002spiking}
W.~Gerstner and W.~M. Kistler, \emph{Spiking neuron models: Single neurons, populations, plasticity}.\hskip 1em plus 0.5em minus 0.4em\relax Cambridge University Press, 2002.

\bibitem{beranek2012acoustics}
L.~L. Beranek and T.~D. Mellow, \emph{Acoustics}.\hskip 1em plus 0.5em minus 0.4em\relax New York: Acoustical Society of America, 2012.

\bibitem{cont2001empirical}
R.~Cont, ``Empirical properties of asset returns: stylized facts and statistical issues,'' \emph{Quantitative Finance}, vol.~1, no.~2, pp. 223--236, 2001.

\bibitem{manabe1967thermal}
S.~Manabe and R.~T. Wetherald, ``Thermal equilibrium of the atmosphere with a given distribution of relative humidity,'' \emph{Journal of the Atmospheric Sciences}, vol.~24, no.~3, pp. 241--259, 1967.

\bibitem{siciliano2010robotics}
B.~Siciliano, L.~Sciavicco, L.~Villani, and G.~Oriolo, \emph{Robotics: Modelling, Planning and Control}.\hskip 1em plus 0.5em minus 0.4em\relax Springer Science \& Business Media, 2010.

\bibitem{kundur1994power}
P.~Kundur, \emph{Power system stability and control}.\hskip 1em plus 0.5em minus 0.4em\relax McGraw-Hill, 1994.

\bibitem{laidler2004chemical}
K.~J. Laidler, \emph{Chemical Kinetics}, 3rd~ed.\hskip 1em plus 0.5em minus 0.4em\relax Pearson, 1987.

\bibitem{wolpert2011principles}
D.~M. Wolpert, J.~Diedrichsen, and J.~R. Flanagan, ``Principles of sensorimotor learning,'' \emph{Nature Reviews Neuroscience}, vol.~12, no.~12, pp. 739--751, 2011.

\bibitem{agrawal2012fiber}
G.~P. Agrawal, \emph{Fiber-optic communication systems}.\hskip 1em plus 0.5em minus 0.4em\relax John Wiley \& Sons, 2012.

\bibitem{pols2011stellar}
O.~R. Pols, \emph{Stellar Structure and Evolution}.\hskip 1em plus 0.5em minus 0.4em\relax Utrecht University Lecture Notes, 2011.

\bibitem{carroll2017astrophysics}
B.~W. Carroll and D.~A. Ostlie, \emph{An Introduction to Modern Astrophysics}, 2nd~ed.\hskip 1em plus 0.5em minus 0.4em\relax Cambridge, 2017.

\bibitem{root_memory_effects_2004}
D.~E. Root, ``Memory effects in power amplifiers,'' \emph{Agilent Technologies}, 2004.

\bibitem{timescale_enzyme_kinetics}
S.~Shoffner and S.~Schnell, ``{Approaches for the estimation of timescales in nonlinear dynamical systems: Timescale separation in enzyme kinetics as a case study},'' \emph{Mathematical Biosciences}, vol. 287, pp. 122--129, 2017.

\bibitem{timescale_neuronal_dynamics}
W.~Gerstner, W.~M. Kistler, R.~Naud, and L.~Paninski, \emph{Neuronal Dynamics: From Single Neurons to Networks and Models of Cognition}.\hskip 1em plus 0.5em minus 0.4em\relax Cambridge University Press, 2014.

\end{thebibliography}
%%%%%%%%%%%%%%%%%%%%%%%%%%%%%%%%%%%%%%%%%%%%%%%%%%%%%%%%%%%%%%%%%%%%%%%%%
%Section
%%%%%%%%%%%%%%%%%%%%%%%%%%%%%%%%%%%%%%%%%%%%%%%%%%%%%%%%%%%%%%%%%%%%%%%%%
%\newpage
\section*{Biography}
 \begin{IEEEbiography}
 	[{\includegraphics[width=1in,height=1.25in, clip,keepaspectratio]	{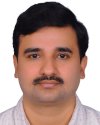}}]{S.~S.~Krishna Chaitanya BULUSU} (M'15--SM'21) received the B.Tech. degree in electronics and communications engineering from Pondicherry Engineering College (PEC), Pondicherry, India, in 2005. He was a recipient of the Thales-French Ministry of External Affairs (MAE) scholarship in 2007. He received the M.S. degree in mobile communications from T{\'e}l{\'e}com ParisTech/Institute Eurecom, Sophia Antipolis, France, in 2009. He received the Ph.D. degree in radiocommunications from Conservatoire National des Arts et Métiers (CNAM), Paris, France, in 2016. From 2016 to 2017, he held a post-doctoral position with the National Institute of Applied Sciences, Rennes, France. From 2018 to 2021, he was an Assistant Professor at École Centrale School of Engineering, Hyderabad, India. Since 2022, he has been a post-doctoral researcher at the Center for Wireless Communications (CWC), University of Oulu, Oulu, Finland. His research interests include digital signal processing and wireless communications for 5G and beyond with a special focus on millimeter waves and RF impairment mitigation. He is a senior member of IEEE and also a fellow of the Institution of Engineers (India). He is the author of technical papers in 18 international conferences and 7 international journal articles and holds an international patent. He was involved in a few European, French, and Finnish research projects involving 6G, 5G, broadcasting systems, and PMR systems.
 \end{IEEEbiography}
 \vskip -2\baselineskip plus -1fil
 \begin{IEEEbiography}
 	[{\includegraphics[width=1in,height=1.25in, clip,keepaspectratio]	{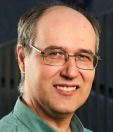}}]{Mikko J. SILLANPÄÄ} received his M.Sc. in Applied Mathematics from University of Jyväskylä 1992 and Ph.D.in Biometry from the University of Helsinki 2000, Finland. He worked as an Academy Research Fellow from 2003-2008 and University Lecturer at the University of Helsinki from 2009-2011, Finland. Since 2011 he has been a full-time Professor in Statistics and is currently also the Department Chair at the Research Unit of Mathematical Sciences, University of Oulu, Finland. He has a long and acknowledged experience in Bayesian statistics and computational methods in high-dimensional problems in biology, medicine, and other fields. Sillanpää has 127 journal publications, 22 conference publications, and book chapters; his works have been cited 4847 times; with h-index 28 (Web of Science). He is an Associate Editor in Theoretical and Applied Genetics, Genetics as well as Editor-in-Chief of the Scandinavian Journal of Statistics. He is a Director of the University of Oulu HiDyn program 2019-2025 funded by the Academy of Finland PROFI5: multidisciplinary data science and artificial intelligence hub strengthening all of the University of Oulu’s profiling areas and the 6G Flagship.
 \end{IEEEbiography}
 %\vskip -2\baselineskip plus -1fil
\end{document}